\documentclass[11pt]{article}

\usepackage[utf8]{inputenc}
\usepackage{microtype}
\usepackage{amsmath}
\usepackage{amssymb}
\usepackage{amsthm}
\usepackage{bm}
\usepackage{dsfont}
\usepackage{authblk}
\usepackage{fullpage}
\usepackage{comment}
\usepackage{nicefrac}
\usepackage{tikz}
\usetikzlibrary{positioning}
\usepackage{mathtools}
\usepackage{bbm}
\usepackage{float}
\usepackage{array,booktabs}
\usepackage[
    backend=biber,
    style=alphabetic,
    maxcitenames=10,
    maxbibnames=99,
    natbib=true,
    url=false,
    doi=false, 
    isbn=false]{biblatex}
\usepackage{hyperref}
\usepackage{mleftright}
\usepackage{thm-restate}
\usepackage[capitalize,noabbrev]{cleveref}
\usepackage[shortlabels]{enumitem}

\newcounter{qst}
\crefname{qst}{Question}{Questions}

\usepackage[ruled,vlined,linesnumbered,noend]{algorithm2e}
\makeatletter
\patchcmd\algocf@Vline{\vrule}{\vrule \kern-0.4pt}{}{}
\patchcmd\algocf@Vsline{\vrule}{\vrule \kern-0.4pt}{}{}
\makeatother

\SetKwComment{Hline}{}{\vspace{-3mm}\textcolor{gray}{\hrule}\vspace{1mm}}
\SetKwComment{HlineWd}{}{\vspace{-3mm}\textcolor{gray}{\hrule width 7.5cm}\vspace{1mm}}
\definecolor{darkgrey}{gray}{0.3}
\definecolor{commentcolor}{gray}{0.5}
\SetKwComment{Comment}{\color{commentcolor}[$\triangleright$\ }{}
\SetCommentSty{}
\SetNlSty{}{\color{darkgrey}}{}
\SetKwProg{Fn}{function}{}{}
\SetKwProg{Subr}{subroutine}{}{}
\crefalias{AlgoLine}{line}%
\crefname{algocf}{Algorithm}{Algorithms}

\makeatletter
\let\cref@old@stepcounter\stepcounter
\def\stepcounter#1{%
  \cref@old@stepcounter{#1}%
  \cref@constructprefix{#1}{\cref@result}%
  \@ifundefined{cref@#1@alias}%
    {\def\@tempa{#1}}%
    {\def\@tempa{\csname cref@#1@alias\endcsname}}%
  \protected@edef\cref@currentlabel{%
    [\@tempa][\arabic{#1}][\cref@result]%
    \csname p@#1\endcsname\csname the#1\endcsname}}
\makeatother

\newcommand{\declarecolor}[2]{\definecolor{#1}{RGB}{#2}\expandafter\newcommand\csname #1\endcsname[1]{\textcolor{#1}{##1}}}
\declarecolor{White}{255, 255, 255}
\declarecolor{Black}{0, 0, 0}
\declarecolor{Maroon}{128, 0, 0}
\declarecolor{Coral}{255, 127, 80}
\declarecolor{Red}{182, 21, 21}
\declarecolor{LimeGreen}{50, 205, 50}
\declarecolor{DarkGreen}{0, 100, 0}
\declarecolor{Navy}{0, 0, 128}
\hypersetup{
	colorlinks=true,
	pdfpagemode=UseNone,
	citecolor=DarkGreen,
	linkcolor=Navy,
	urlcolor=Navy,
	pdfstartview=FitW
}

\theoremstyle{plain}
\newtheorem{theorem}{Theorem}[section]
\newtheorem{lemma}[theorem]{Lemma}
\newtheorem{corollary}[theorem]{Corollary}
\newtheorem{proposition}[theorem]{Proposition}
\newtheorem{fact}[theorem]{Fact}

\newtheorem{claim}[theorem]{Claim}

\newtheorem{property}[theorem]{Property}

\theoremstyle{definition}
\newtheorem{definition}[theorem]{Definition}

\theoremstyle{remark}

\newcommand*{\N}{{\mathbb{N}}}
\newcommand*{\Z}{{\mathbb{Z}}}
\let\R\relax
\newcommand*{\R}{{\mathbb{R}}}

\newcommand{\defeq}{\coloneqq}

\let\poly\relax
\let\polylog\relax
\let\co\relax

\DeclareMathOperator{\poly}{poly}
\DeclareMathOperator{\reg}{Reg}
\DeclareMathOperator{\co}{co}
\DeclareMathOperator{\var}{Var}
\DeclareMathOperator{\polylog}{polylog}

\newcommand{\tree}{\mathcal{T}}
\NewDocumentCommand{\treeset}{o}{\mathbb{T}\IfNoValueF{#1}{_{#1}}}
\newcommand{\intphi}{\Phi^\mathrm{int}}
\newcommand{\obsut}{\textsc{ObserveUtility}}
\newcommand{\nextstr}{\textsc{NextStrategy}}
\newcommand{\fp}{\textsc{StationDistrib}}
\newcommand{\slomwu}{\textsc{SL-OMWU}}
\newcommand{\bmomwu}{\textsc{BM-OMWU}}
\newcommand{\intreg}{\mathrm{IntReg}}
\newcommand{\swapreg}{\mathrm{SwapReg}}

\makeatletter
\AtBeginDocument{%
  \@ifpackageloaded{amsmath}{%
    \newcommand*\patchAmsMathEnvironmentForLineno[1]{%
      \expandafter\let\csname old#1\expandafter\endcsname\csname #1\endcsname
      \expandafter\let\csname oldend#1\expandafter\endcsname\csname end#1\endcsname
      \renewenvironment{#1}%
                       {\linenomath\csname old#1\endcsname}%
                       {\csname oldend#1\endcsname\endlinenomath}%
    }%
    \newcommand*\patchBothAmsMathEnvironmentsForLineno[1]{%
      \patchAmsMathEnvironmentForLineno{#1}%
      \patchAmsMathEnvironmentForLineno{#1*}%
    }%
    \patchBothAmsMathEnvironmentsForLineno{equation}%
    \patchBothAmsMathEnvironmentsForLineno{align}%
    \patchBothAmsMathEnvironmentsForLineno{flalign}%
    \patchBothAmsMathEnvironmentsForLineno{alignat}%
    \patchBothAmsMathEnvironmentsForLineno{gather}%
    \patchBothAmsMathEnvironmentsForLineno{multline}%
  }{}
}
\makeatother

\renewcommand{\vec}[1]{\bm{#1}}
\newcommand{\mat}[1]{\mathbf{#1}}

\newcommand{\range}[1]{[\![#1]\!]}
\newcommand{\edge}[2]{#1\!\to\!#2}
\newcommand{\bigL}{\vec{\mathcal{L}}}

\DeclarePairedDelimiterX{\card}[1]{\lvert}{\rvert}{#1}
\DeclarePairedDelimiterX{\abs}[1]{\lvert}{\rvert}{#1}
\DeclarePairedDelimiterX{\norm}[1]{\lVert}{\rVert}{#1}
\DeclarePairedDelimiterX{\tuple}[1]{\lparen}{\rparen}{#1}
\DeclarePairedDelimiterX{\parens}[1]{\lparen}{\rparen}{#1}
\DeclarePairedDelimiterX{\brackets}[1]{\lbrack}{\rbrack}{#1}
\DeclarePairedDelimiterX{\set}[1]\{\}{#1}
\let\Pr\relax
\DeclarePairedDelimiterXPP{\Pr}[1]{\mathbb{P}}[]{}{#1}
\DeclarePairedDelimiterXPP{\PrX}[2]{\mathbb{P}_{#1}}[]{}{#2}
\DeclarePairedDelimiterXPP{\Ex}[1]{\mathbb{E}}[]{}{#1}
\DeclarePairedDelimiterXPP{\ExX}[2]{\mathbb{E}_{#1}}[]{}{#2}

\usepackage{lineno}

\usetikzlibrary{fit,shapes.misc,arrows.meta}
\makeatletter
\tikzset{
  fitting node/.style={
    inner sep=0pt,
    fill=none,
    draw=none,
    reset transform,
    fit={(\pgf@pathminx,\pgf@pathminy) (\pgf@pathmaxx,\pgf@pathmaxy)}
  },
  reset transform/.code={\pgftransformreset}
}
\tikzset{cross/.style={path picture={
  \draw[black]
(path picture bounding box.south east) -- (path picture bounding box.north west) (path picture bounding box.south west) -- (path picture bounding box.north east);
}}}
\tikzstyle{ox}=[semithick,draw=black,circle,cross,inner sep=1.2mm]
\makeatother

\newcommand{\nc}{\newcommand}
\newcount\Comments
\nc\noah[1]{\ifnum\Comments=1 {\textcolor{purple}{[ng: #1]}}\fi}
\nc\maxfish[1]{\ifnum\Comments=1{\textcolor{blue}{[mf: #1]}}\fi}
\nc\costis[1]{\ifnum\Comments=1{\textcolor{brown}{[cd: #1]}}\fi}
\nc\costiss[1]{\textcolor{red}{#1}}

\nc{\Opthedge}{OMWU\xspace}

\nc{\DMO}{\DeclareMathOperator}
\nc\old[1]{\textcolor{brown}{[old: #1]}}
\nc{\BR}{\mathbb{R}}
\nc{\BC}{\mathbb{C}}
\DMO{\Bin}{Bin}
\nc{\BN}{\mathbb{N}}
\nc{\distrs}[1]{\Delta({#1})}
\nc{\BZ}{\mathbb{Z}}
\nc{\ep}{\epsilon}
\nc{\ra}{\rightarrow}
\nc{\st}{\star}
\nc{\Reg}[2]{\REG_{{#1},{#2}}}
\nc{\til}{\tilde}
\nc{\kld}[2]{\KL({#1};{#2})}
\nc{\chisq}[2]{\chi^2({#1};{#2})}
\DMO{\POLYLOG}{polylog}
\nc{\matx}[1]{\left(\begin{matrix}#1\end{matrix}\right)}
\DMO{\VAR}{Var}
\DMO{\COV}{Cov}
\nc{\Var}[2]{\VAR_{{#1}}\left({#2}\right)}
\nc{\Cov}[3]{\COV_{{#1}}\left({#2},{#3}\right)}
\DMO{\DD}{D}
\nc{\fd}[2]{\DD_{#1}{#2}}
\nc{\fds}[3]{\left(\fd{#1}{#2}\right)\^{#3}}
\nc{\fdc}[2]{\DD^\circ_{#1}{#2}}
\nc{\fdcs}[3]{\left(\fdc{#1}{#2}\right)\^{#3}}
\nc{\shf}[2]{\EEE_{#1}{#2}}
\nc{\shfs}[3]{\left(\shf{#1}{#2}\right)\^{#3}}
\nc{\normst}[2]{\left\| {#2} \right\|_{#1}^\st}
\renewcommand{\^}[1]{^{(#1)}}

\nc{\grad}{\nabla}
\nc{\lng}{\langle}
\nc{\rng}{\rangle}
\nc{\bbone}{\mathbf{1}}
\nc{\bbzero}{\mathbf{0}}
\nc{\MD}{\mathcal{D}}
\nc{\MM}{\mathcal{M}}
\nc{\MZ}{\mathcal{Z}}
\nc{\MU}{\mathcal{U}}
\nc{\MC}{\mathcal{C}}
\nc{\MT}{\mathbb{T}^{n}}
\nc{\MS}{\mathcal{S}}
\nc{\MX}{\mathcal{X}}
\nc{\MY}{\mathcal{Y}}
\nc{\MB}{\mathcal{B}}
\nc{\MJ}{\mathcal{J}}
\nc{\MF}{\mathcal{F}}
\nc{\MG}{\mathcal{G}}
\nc{\MR}{\mathcal{R}}
\nc{\ML}{\mathcal{L}}
\nc{\MQ}{\mathcal{Q}}
\newcommand{\paren}[1]{\left({#1}\right)}
\nc{\ba}{\mathbf{A}}
\nc{\bx}{\mathbf{x}}
\nc{\by}{\mathbf{y}}
\nc{\bz}{\mathbf{z}}
\nc{\bs}{\mathbf{s}}
\nc{\bt}{\mathbf{t}}
\nc{\br}{\mathbf{r}}

\nc{\ME}{\mathcal{E}}
\DMO{\View}{View}
\DMO{\KL}{KL}
\nc{\MW}{\mathcal{W}}
\nc{\CS}{\mathscr{S}}
\nc{\CI}{\mathscr{I}}
\nc{\CQ}{\mathscr{Q}}
\nc{\CL}{\mathscr{L}}
\nc{\CM}{\mathscr{M}}
\nc{\CG}{\mathscr{G}}
\nc{\CR}{\mathscr{R}}

\nc{\wh}{\widehat}

\newcommand{\ps}[1]{\left[{#1}\right]}
\newcommand{\bigcard}[1]{\left|#1\right|}

\nc{\BM}{BM\xspace}
\nc{\ALG}{\texttt{ALG}}
\nc{\MCT}{{\rm MCT}}

\nc{\matrixLL}{\mat{L}}
\nc{\vectorecks}{\vec{x}}
\nc{\vectorLL}{\vec{\ell}}
\nc{\matrixKYU}{\mat{Q}}
\nc{\rowdot}{\cdot}

\bibliography{main}

\title{Near-Optimal No-Regret Learning for Correlated Equilibria in Multi-Player General-Sum Games}
\author{
    Ioannis Anagnostides$^\dagger$ \quad
    Constantinos Daskalakis$^\ddagger$ \quad
    Gabriele Farina$^\dagger$ \\
    Maxwell Fishelson$^\ddagger$ \quad
    Noah Golowich$^\ddagger$ \quad 
    Tuomas Sandholm$^{\dagger,\S,\P,\#}$\\\vspace{6mm}
    $^\dagger$\ Carnegie Mellon University, Computer Science Department \\
    $^\ddagger$\ MIT CSAIL\\
    $^\S$\ Strategy Robot, Inc.\\
    $^\P$\ Optimized Markets, Inc.\\
    $^\#$\ Strategic Machine, Inc.\\\vspace{2mm}
    {\small\texttt{\{ianagnos,gfarina,sandholm\}@cs.cmu.edu}, \quad \texttt{\{costis,maxfish,nzg\}@csail.mit.edu}}
}

\date{November 2021}

\begin{document}

\maketitle

\pagenumbering{gobble}

\begin{abstract}
    Recently, Daskalakis, Fishelson, and Golowich (\cite{Daskalakis21:Near} NeurIPS\,`21) showed that if all agents in a multi-player general-sum normal-form game employ Optimistic Multiplicative Weights Update (OMWU), the external regret of every player is  $O(\polylog(T))$ after $T$ repetitions of the game. In this paper we extend their result from external regret to internal and swap regret, thereby establishing uncoupled learning dynamics that converge to an approximate correlated equilibrium at the rate of $\widetilde{O}\left( T^{-1} \right)$. This substantially improves over the prior best rate of convergence for correlated equilibria of $O(T^{-3/4})$ due to Chen and Peng (\cite{Chen20:Hedging} NeurIPS\,`20), and it is optimal up to polylogarithmic factors in $T$. 
    
    To obtain these results, we develop new techniques for establishing higher-order smoothness for learning dynamics involving fixed point operations. Specifically, we first establish that the no-internal-regret learning dynamics of Stoltz and Lugosi (\cite{Stoltz05:Internal} Mach Learn\,`05) are equivalently simulated by no-external-regret dynamics on a combinatorial space. This allows us to trade the computation of the stationary distribution on a polynomial-sized Markov chain for a (much more well-behaved) linear transformation on an exponential-sized set, enabling us to leverage similar techniques as \cite{Daskalakis21:Near} to near-optimally bound the internal regret. 
    
    Moreover, we establish an $O(\polylog(T))$ no-swap-regret bound for the classic algorithm of Blum and Mansour (\cite{Blum07:From} JMLR\,`07). We do so by introducing a technique based on the Cauchy Integral Formula that circumvents the more limited combinatorial arguments of \cite{Daskalakis21:Near}. In addition to shedding clarity on the near-optimal regret guarantees of \cite{Blum07:From}, our arguments provide insights into the various ways in which the techniques by \cite{Daskalakis21:Near} can be extended and leveraged in the analysis of more involved learning algorithms.
\end{abstract}

\clearpage
\pagenumbering{arabic}

\section{Introduction}

Online learning and game theory share an intricately connected history tracing back to Robinson's analysis of \emph{fictitious play} \citep{Robinson51:iterative}, as well as Blackwell's seminal \emph{approachability theorem} \citep{Blackwell56:analog}, which served as the advent of the modern \emph{no-regret} framework \citep{Hart00:Simple,Abernethy11:Blackwell}. These connections have since led to the discovery of broad learning paradigms such as Online Mirror Descent, encompassing algorithms such as the celebrated Multiplicative Weights Update (MWU) \citep{Littlestone94:Weighted}. Importantly, \emph{uncoupled} learning dynamics overcome the often unreasonable assumption that players have perfect knowledge of the game, while they have also emerged as a central component in several recent landmark results in computational game solving \citep{Brown17:Superhuman,Moravvcik17:DeepStack}. Moreover, another compelling feature of the no-regret framework is that it guarantees robustness even against \emph{adversarial} opponents. Indeed, there are broad families of learning paradigms \citep{Shalev-Shwartz12:Online} that accumulate $O(\sqrt{T})$ regret after $T$ iterations, a barrier which is known to be insuperable in fully adversarial environments \citep{Cesa-Bianchi06:Prediction}. However, this begs the question: \emph{What if players do not face adversarial losses, but instead face \emph{predictable} losses?}

This question was first addressed by \citet{Daskalakis11:Near}. They devised a decentralized variant of Nesterov's \emph{excessive gap technique} \citep{Nesterov05:Smooth}, enjoying a near-optimal rate of convergence of $O(\log T/ T)$ to Nash equilibrium when employed by \emph{both} players in a two-player zero-sum normal-form game. (For brevity we will henceforth omit the specification ``normal-form'' when referring to games.) At the same time, their algorithm also guarantees optimal (external) regret under worst-case losses.
Subsequently, \citet{Rakhlin13:Optimization,Rakhlin13:Online} introduced an \emph{optimistic} variant of Online Mirror Descent---considerably simpler than the algorithm proposed in \citep{Daskalakis11:Near}---achieving optimal convergence rate to Nash equilibrium, again in zero-sum games. Then \citet{Syrgkanis15:Fast} identified a broad class of \emph{predictive} learning algorithms that induce no-regret learning dynamics in multi-player general-sum games that guarantee $O(T^{1/4})$ regret if followed by each player. 
This line of work culminated in a recent advancement by \citet{Daskalakis21:Near}, where it was shown that, when all players employ an optimistic variant of MWU, each player incurs only $O(\polylog(T))$ regret. In turn, this implies that the average product distribution of play induced by optimistic MWU is an $\widetilde{O}\left( T^{-1} \right)$-approximate\footnote{As usual, we use the notation $\widetilde{O}(\cdot)$ to suppress polylogarithmic factors of $T$. Also note that for simplicity, and with a slight abuse of notation, in our introductory section we use the big-$O$ notation to hide game-specific parameters.} \emph{coarse correlated equilibrium} (\emph{CCE}) after $T$ repetitions of the game.

Yet, it is well-understood that a CCE prescribes a rather weak notion of equilibrium~\cite{Gordon08:No}. An arguably more compelling solution concept\footnote{In general-sum multi-player games it is typical to search for solution concepts more permissive than \emph{Nash equilibria} \citep{Nash50:Equilibrium} as the latter is known to be computationally intractable under reasonable assumptions \citep{Daskalakis09:Complexity,Chen09:Settling,Etessami07:Complexity, Rubinstein16:Settling,Babichenko17:Communication}.} in multi-player general-sum games is that of \emph{correlated equilibria} (\emph{CE}) \citep{Aumann74:Subjectivity}. Like CCE, it is known that CE can be computed through uncoupled learning dynamics. Thus, our paper is concerned with the following central question:\\[3mm]
\begin{tikzpicture}
    \node[text width=15cm,align=center,inner sep=0pt] at (0,0) {\textit{Are there learning dynamics that, if followed by all players in a multi-player general-sum game, guarantee convergence with rate $\widetilde{O}\left( T^{-1} \right)$ to a correlated equilibrium?}};
    \node at (8.0cm,0) {\refstepcounter{qst}\label{qst:1}($\spadesuit$)};
    \node at (-8cm, 0) {};
\end{tikzpicture}
The main contribution of our paper is to answer this question in the affirmative.
Unlike in the case of CCE, typical no-external-regret dynamics such as MWU are known not to guarantee convergence to CE. Instead, specialized \emph{no-internal-regret} or \emph{no-swap-regret} algorithms have to be employed to converge to CE \citep{Cesa-Bianchi06:Prediction}. Compared to no-external-regret dynamics, these learning dynamics are considerably more complex in that all known algorithms require the computation of the stationary distribution of a certain Markov chain at every iteration.
%Leveraging the techniques in \citep{Daskalakis21:Near} inevitably requires a very refined characterization of the dynamics using higher-order smoothness.
%
Our main primary technical contribution is to develop techniques to overcome these additional challenges. %induced by stationary distributions.

%Nonetheless, our main contribution is to answer this question in the affirmative. One of our key observations is a novel viewpoint on the well-known reduction from \emph{external} to \emph{internal} regret, which turns out to unlock the problem.

\subsection{Contributions}

Our work presents a refined analysis of the no-internal-regret algorithm of \citet{Stoltz05:Internal}, as well as the no-swap-regret algorithm of \citet{Blum07:From}, both instantiated with Optimistic Multiplicative Weights Update (OMWU). Going forward, we will refer to those learning dynamics as $\slomwu$ and $\bmomwu$, respectively. Our primary contribution is to show that both of these algorithms exhibit a near-optimal convergence rate of $\widetilde{O}(T^{-1})$, settling \Cref{qst:1} in the affirmative. More precisely, for $\slomwu$ our main theorem is summarized as follows.

\begin{restatable}{theorem}{intmain}
    \label{theorem:main summary}
Consider a general-sum multi-player game with $m$ players, with each player $i \in \range{m}$ having $n_i$ actions. There exists a universal constant $C > 0$ such that, when all players select strategies according to algorithm $\slomwu$ with step size $\eta = 1/(C\cdot m \log^4 T)$, the internal regret of every player $i \in \range{m}$ is bounded by $O \mleft( m \log n_i \log^4 T \mright)$. As a result, the average product distribution of play is an $O \mleft( {(m \log n \log^4 T)}/{T} \mright)$-approximate correlated equilibrium.
\end{restatable}

This matches, up to constant factors, the rate of convergence for coarse correlated equilibria as follows by the result in \citep{Daskalakis21:Near}, and it is optimal, within the no-regret framework,%
%\footnote{The computation of a correlated equilibrium can also be expressed as a linear program, for which several polynomial-time algorithms with a linear rate of convergence are known in the literature. However, those techniques do not provide uncoupled learning dynamics, and are known to scale poorly in large games.}
\footnote{Finding a correlated equilibrium can be phrased as a linear programming problem, and thus $\epsilon$-approximate correlated equilibria can be found in time $\poly(m, n, \log(1/\epsilon))$, where $n = \max_{i \in \range{m}} \{ n_i \}$, for succinct multi-player games~\citep{Papadimitriou08:Computing}. However, the procedure for doing so, \emph{ellipsoid against hope}, cannot be phrased as uncoupled dynamics and is unlikely to be run by players competing in a repeated game.}
up to polylogarithmic factors~\citep{Daskalakis11:Near}. 
%Remove "as a corollary"???
% As a corollary, t
This also substantially improves upon the $O(T^{-3/4})$ rate of convergence for correlated equilibria recently shown by \citet{Chen20:Hedging}. Moreover, since swap regret on an $n$-simplex is trivially at most $n$ times larger than internal regret (\emph{e.g.}, see \citet[pp. 1311]{Blum07:From}),
%Isn't the above obvious? If so, should say so, e.g., using "trivially". ??? DONE
\cref{theorem:main summary} directly gives a bound in terms of swap regret as well, stated as follows.
\begin{corollary}
    \label{cor:int-to-swap}
    If all players select strategies according to algorithm \slomwu{}, the swap regret of every player $i \in \range{m}$ is bounded by $O \mleft( m\, n_i \log n_i \log^4 T \mright)$.
\end{corollary}

For the popular and more involved algorithm $\bmomwu$, our main theorem is summarized as follows.
\begin{restatable}{theorem}{thmbm}%[Swap Regret of BM]
      \label{thm:bound-l-dl}
    Consider a general-sum multi-player game with $m$ players, with each player $i \in \range{m}$ having $n_i$ actions. There exists a universal constant $C > 0$ such that, when all players select strategies according to algorithm $\bmomwu$ with step size $\eta =1/(C \cdot m\,n_i^3\log^4(T))$, the swap regret of every player $i \in \range{m}$ is bounded by $O(m\,n_i^4 \log n_i \log^4(T))$.
    As a result, the average product distribution of play is an $O \mleft( {(m\, n^4 \log n \log^4 T)}/{T} \mright)$-approximate correlated equilibrium.
\end{restatable}

Finally, we remark that $\slomwu$ and $\bmomwu$ instantiated with the learning rates described in \cref{theorem:main summary,thm:bound-l-dl} guarantee near-optimal swap regret (in $T$) when all players use the same dynamics, but might not against general, adversarial losses. 
%Not clear how the next sentence relates to the sentence above???
To guarantee near-optimal swap regret in both the adversarial and the non-adversarial regime, an adaptive choice of learning rate similar to that in \citep{Daskalakis21:Near} can be used (see \Cref{corollary:adversarial}).
%This implies the existence of learning dynamics with near-optimal swap regret in both the adversarial and the non-adversarial regime.
%Here it is implied that the transition probabilities of the underlying Markov chains are updated using OMWU (see \cref{section:prelim} for an overview of the construction).
% \todo{We shuold summarize the results about Blum \& Mansour}

\subsection{Overview of Techniques}

The recent work of \citet{Daskalakis21:Near} identified \emph{higher-order smoothness} of no-external-regret learning dynamics as a key property for obtaining near-optimal external regret bounds. In particular, they showed that for the no-external-regret dynamics OMWU, the higher-order differences (\Cref{definition:fin-dif}) of the sequence of loss vectors decay exponentially at orders up to roughly $\log T$. However, establishing such higher-order smoothness for no-internal- and no-swap-regret learning dynamics is a considerable challenge since the known algorithms involve computing the stationary distribution of a certain Markov chain at every iteration. Our main technical contribution is to develop new techniques to effectively address this challenge.

\paragraph{Proof of \cref{theorem:main summary}: Analyzing \slomwu.}
%In order to establish \Cref{theorem:main summary}, one of our key insights is that when the underlying regret minimizers are updated using (O)MWU,
First, we show that internal regret minimization on an $n$-simplex can be simulated by no-external-regret dynamics on the combinatorial space of all $n$-node \emph{directed trees} (\cref{theorem:equivalence}). Our equivalence result enables us to trade the computation of a stationary distribution of a polynomial-sized Markov chain for a (much more well-behaved) linear transformation on an exponential-sized set. To our knowledge, this is the first no-internal-to-no-external-regret reduction that sidesteps the computation of stationary distributions of Markov chains, and might have applications beyond the characterization of higher-order smoothness of the dynamics. Based on our equivalence result, we then adapt and leverage the known higher-order smoothness techniques for no-external-regret dynamics~\citep{Daskalakis21:Near}. We stress that our analysis is eventually brought back to a ``low-dimensional'' regret minimizer, instead of solely operating over the space of directed trees; this step is crucial for obtaining the logarithmic dependence on the number of actions of each player for no-internal-regret dynamics (\cref{theorem:main summary}).

%This enables us to essentially reduce the analysis to bounding the \emph{external regret} of an algorithm observing OMWU-type losses, which can be performed by leveraging the techniques in \citep{Daskalakis21:Near}. %Nonetheless, we point out that this does not quite give an ``end-to-end'' reduction since the sequence of observed losses possesses a slightly different structure than that encountered in the usual setting. We also remark that this analysis is performed with respect to a ``low-dimensional'' regret minimizer, and not based on the one operating over the space of all directed trees; this is crucial in order to obtain tight bounds in terms of $n_i$, as the latter choice could induce a substantial overhead.
% We remark that our reduction to external regret on an exponentially large combinatorial space is a purely analytic tool. 
%
The equivalence result mentioned in the previous paragraph arises as a consequence of the classic Markov chain tree theorem, which provides a closed-form combinatorial formula for the stationary distribution of an ergodic Markov chain, and crucially relies on the multiplicative structure of the update rule of (O)MWU. 
Specifically, we prove that the stationary distributions of certain Markov chains whose transition probabilities are updated through (O)MWU are themselves linear transformations of iterates produced by (O)MWU.
Our equivalence gives a direct way to argue about the higher-order smoothness of stationary distributions of Markov chains, substantially extending the first-order smoothness observation of \citet{Chen20:Hedging}. Furthermore, we expect the equivalence to continue to hold beyond stationary distributions of Markov chains, to the more general problem of computing \emph{fixed points} of linear transformations required in the framework of \emph{Phi-regret} \citep{Stoltz07:Learning,Greenwald03:General,Farina21:Simple}.

\paragraph{Proof of \cref{thm:bound-l-dl}: Analyzing \bmomwu.}
The techniques we described so far enable us to establish the near-optimal internal and swap regret bounds for $\slomwu$ (\Cref{theorem:main summary} and \Cref{cor:int-to-swap}), as well as the corresponding convergence to correlated equilibrium. However, different techniques are necessary to establish the regret bound of $\bmomwu$ (\cref{thm:bound-l-dl}). At a high level, $\bmomwu$ runs $n_i$ independent external regret minimizers for each player $i$, aggregates the outputs into a transition matrix of a Markov chain, and then computes its stationary distribution. %Because of the independence between the $n_i$ regret minimizers, it is unclear if an analogue of the simulation result (\cref{theorem:equivalence}) holds. 

% Thus, r
Rather than arguing indirectly in terms of a supplementary external-regret minimizer, we \emph{directly} analyze the higher-order smoothness of a sequence of stationary distributions of Markov chains, at the cost of ultimately obtaining a worse dependence on the number of actions $n_i$ in our swap regret bounds. Using the machinery of \cite{Daskalakis21:Near} (in particular, the boundedness chain rule of \cref{lemma:boundedness-chain-rule}), doing so boils down to obtaining a bound on the Taylor series coefficients of the function that maps the entries of an ergodic matrix $\mat{Q}$ to $\mat{Q}$'s stationary distribution. Taken literally, such a bound is not quite possible, since the stationary distribution may have singularities around the non-ergodic matrices. However, we show that by using the Markov chain tree theorem together with the multi-dimensional version of Cauchy's integral formula, it is possible to bound the Taylor series of the function mapping the \emph{logarithms} of the entries of $\mat{Q}$ to its stationary distribution (see \Cref{lem:mct-taylor}). Leveraging the exponential-type structure of the OMWU updates, we then use this bound to obtain the desired guarantee on the higher-order differences of the stationary distributions (\Cref{lem:dh-bound}). 
% \noah{let me know if you think I should add more, the rest of the proof follows pretty closely our previous paper}
% \todo{In light of this, we develop more robust techniques in order to show that $\bmomwu$ also attains a near-optimal rate of $\widetilde{O}(T^{-1})$.} \todo{What techniques???}

\subsection{Further Related Work}

No-internal-regret algorithms that require black-box access to a \emph{single} no-external-regret minimizer are known in the literature~\citep{Stoltz05:Internal,Cesa-Bianchi06:Prediction}. This is in contrast with the construction of \citet{Blum07:From} for the stronger notion of swap regret, which requires $n$ independent no-external-regret minimizers---one per each action of the player. Nevertheless, both classes of algorithms involve computing the stationary distribution of a certain Markov chain at every iteration. The intrinsic complexity associated with the computation of a stationary distribution was arguably the main factor limiting our ability to give accelerated convergence guarantees for either class of algorithms. Indeed, while learning dynamics guaranteeing external regret bounded by $O(T^{1/4})$ have been known for several years \citep{Syrgkanis15:Fast}, a matching bound for swap regret was only recently shown by \citet{Chen20:Hedging}. 
%
%The technique used by \citet{Chen20:Hedging} revolves around \emph{multiplicative stability}, a weaker concept than the higher-order stability leveraged by \citep{Daskalakis21:Near} and in our paper.

The setting studied in our paper (learning dynamics for correlated equilibrium) is substantially more challenging than the problem of giving accelerated learning dynamics for Nash equilibria in two-player zero-sum games, as well learning dynamics for \emph{smooth games} \citep{Roughgarden15:Intrinsic}. Indeed, while in our setting the convergence to the equilibrium is driven by the \emph{maximum} internal (or swap) regret cumulated by the players, in the latter two settings the quality metric is driven by the \emph{sum} of the external regrets. As shown by \citet{Syrgkanis15:Fast}, it is possible to guarantee a \emph{constant sum} of external regrets under a \emph{broad} class of \emph{predictive} no-regret algorithms which includes optimistic OMD and optimistic FTRL under very general distance-generating functions. This is in contrast with the case of no-regret dynamics for CCE and CE, where it remains an open question to give broad classes of algorithms that can achieve near-optimal convergence. 

Finally, we point out that optimistic variants of FTRL such as OMWU have been shown to also converge in the \emph{last-iterate sense}, and the convergence is known to be linear\footnote{However, the convergence rate of the last-iterate may \emph{not} depend polynomially in the size of the game \citep{Wei21:Linear}.} \citep{Daskalakis18:Training,Daskalakis18:The,Daskalakis19:Last,Wei21:Linear}, but this holds only for restricted classes of games, such as two-player zero-sum games.

%Finally, we remark that finding a correlated equilibrium can be phrased as a linear programming problem, and thus $\epsilon$-correlated equilibria can be found in time $\polylog 1/\epsilon$, even for succinct multi-player games \cite{papadimitriou_computing_2008}. However, the procedure for doing so, \emph{ellipsoid against hope}, cannot be phrased as uncoupled dynamics, scales poorly in large games, and is unlikely to be run by players competing in a repeated game. \noah{delete this/move to footnote earlier???}
\section{Preliminaries}
\label{section:prelim}

%For consistency, we will mostly follow the notation in \citep{Daskalakis21:Near}. Specifically, 
Consider a finite \emph{normal-form} game $\Gamma$ consisting of a set of $m$ players $\range{m} \coloneqq \{1, 2, \dots, m\}$ such that every player $i$ has an \emph{action space} $\mathcal{A}_i \coloneqq \range{n_i}$, for some $n_i \in \N$. The \emph{joint} action space will be represented with $\mathcal{A} \coloneqq \mathcal{A}_1 \times \dots \times \mathcal{A}_m$. For a given action profile $\vec{a} = (a_1, \dots, a_m) \in \mathcal{A}$, the \emph{loss function} $\Lambda_i : \mathcal{A} \to [0, 1]$ specifies the loss of player $i$ under the action profile $\vec{a}$; note that the normalization of the losses comes without any loss of generality. A \emph{mixed strategy} $\vec{x}_i \in \Delta(\mathcal{A}_i)$ for a player $i \in \range{m}$ is a probability distribution over $i$'s action space $\mathcal{A}_i$, so that the coordinate $\vec{x}_i[j]$ indicates the probability that player $i$ will select action $j \in \range{n_i}$. A \emph{deterministic strategy} refers to a mixed strategy supported on a single coordinate. Given a \emph{joint} vector of mixed strategies $\vec{x} = (\vec{x}_1, \dots, \vec{x}_m)$, the \emph{expected} loss $\vec{\ell}_i$ of player $i$ is such that $\vec{\ell}_i[j] \coloneqq \ExX{\vec{a}_{-i} \sim \vec{x}_{-i}}{\Lambda_i(j, \vec{a}_{-i})}$, for all $j \in \range{n_i}$, where we used the notation $\vec{a}_{-i}$ to denote the vector $\vec{a}$ excluding the coordinate corresponding to $i$'s action; i.e. $\vec{a}_{-i} \coloneqq (a_1, \dots, a_{i-1}, a_{i+1}, \dots, a_m)$.

\paragraph{Hindsight Rationality and Regret Minimization.} A standard quality metric in the theory of learning in games is \emph{hindsight rationality}. Hindsight rationality encodes the idea that a player has ``learnt'' to play the game when, looking back at the history of play, there is no transformation of their strategies that---applied to the whole history of play---would have led to strictly better utility for that player. This notion is operationalized through \emph{Phi-regret}. Formally, 
the $\Phi_i$-regret incurred by the sequence of strategies $\vec{x}_i^{(1)}, \dots, \vec{x}_i^{(T)}\in\Delta^{n_i}$ selected by player~$i\in\range{m}$ is defined as 
\begin{equation}
    \label{eq:phi_regret}
    \reg^T_{\Phi_i} \defeq \sum_{t=1}^T \langle \vec{x}_i^{(t)}, \vec{\ell}_i^{(t)} \rangle - \min_{\phi^* \in \Phi_i} \sum_{t=1}^T \langle \phi^* (\vec{x}_i^{(t)}), \vec{\ell}_i^{(t)} \rangle.
\end{equation}
Notable special cases are identified based on the particular choice of transformations $\Phi_i$, as follows.
\begin{itemize}[nosep]
\item[(i)] \emph{External} regret (or simply \emph{regret}) corresponds to the case where $\Phi_i$ is the set of all \emph{constant} functions $\Phi_i^\text{const} \defeq \{\phi : \Delta^{n_i} \to \Delta^{n_i}, \phi(\vec{x}) = \phi(\vec{x}') \ \forall\, \vec{x},\vec{x}'\in\Delta^{n_i}\}$.
\item[(ii)] \emph{Internal} regret corresponds to the set of linear transformation $\Phi_i = \intphi_i$ that transport probability mass from some action $j$ to some other action $k$. Formally, $\intphi_i$ is the convex hull $\intphi_i \coloneqq \co\{ \phi_{j \to k} \}_{j, k \in \mathcal{A}_i, j \neq k}$ of the functions $\phi_{j\to k}$ defined as
\begin{equation}\label{eq:phi jk}
    \phi_{j \to k} : \Delta^{n_i} \to \Delta^{n_i}, \quad \vec{x} \mapsto \vec{x} + (\vec{e}_k - \vec{e}_j)\, \vec{x}[j] = \big(\mat{I} + (\vec{e}_k - \vec{e}_j)\vec{e}_j^\top\big)\,\vec{x} \eqqcolon \mat{E}_{j\to k}\, \vec{x}.
\end{equation}
\item[(iii)] \emph{Swap} regret corresponds to the set $\Phi_i^\text{swap}$ of all linear transformations $\Delta^{n_i}\ni \vec{x} \mapsto \mat{Q}^\top \vec{x}$ where $\mat{Q}$ is a \emph{row-stochastic matrix}, that is, a non-negative matrix whose rows all sum to $1$.
\end{itemize}

\paragraph{Connections with Solution Concepts.}
There exist connections between the different notions of hindsight rationality described above and game-theoretic solution concepts, including correlated equilibria (the focus of this paper), whose definition is recalled next \citep{Foster97:Calibrated,Hart00:Simple}.
\begin{definition}[Correlated Equilibrium]
    A probability distribution $\vec{\mu}$ over $\Pi_1 \times \dots\times\Pi_m$ is said to be \emph{an $\epsilon$-correlated equilibrium}, where $\epsilon > 0$, if for every player $i \in \range{m}$ and every $\phi \in \intphi_i$, 
    \begin{equation}
    \mathbb{E}_{(\vec{\pi}_{1}, \dots, \vec{\pi}_m) \sim \vec{\mu}}\big[\langle \vec{\ell}_i, \vec{\pi}_i\rangle - \langle\vec{\ell}_i, \phi (\vec{\pi}_i) \rangle\big] \leq \epsilon.
    \end{equation}
\end{definition}
A folklore result states that the average product distribution of play of no-internal-regret players converges to an approximate \emph{correlated equilibrium} (\emph{CE}). More precisely, the following holds.
\begin{theorem}
    \label{theorem:convergence_correlated}
    Consider $T$ repetitions of play in a game where every player $i\in\range{m}$ employs an algorithm that produces strategies $\vec{x}_i^{(1)},\dots,\vec{x}_i^{(T)}\in\Delta^{n_i}$ with internal regret $\intreg_i^T$. Then, the average product distribution of play
    $
        \bar{\vec{\mu}} \defeq \frac{1}{T}\sum_{t=1}^T \vec{x}_1^{(t)}\otimes\dots\otimes\vec{x}_m^{(t)}
    $
    is an $O(\max_i \intreg_i^T/T)$-CE. %(In particular, if $\intreg_i^T = o(T)$, then $\bar{\vec{\mu}}$ is an $o(1)$-correlated equilibrium.)
\end{theorem}

\paragraph{Optimistic Multiplicative Weights Update.}
%
% A celebrated result in the theory of online learning is the existence of natural learning algorithms that guarantee that the \emph{average regret} asymptotically vanishes as $T \to \infty$, even under \emph{adversarially selected} loss functions \citep{Cesa-Bianchi06:Prediction}.
%
The \emph{optimistic} variant of MWU, called \emph{Optimistic Multiplicative Weights Update (OMWU)}\footnote{The OMWU algorithm is sometimes referred to as \emph{Optimistic Hedge} in the literature.} is a particular instantiation of the more general \emph{optimistic FTRL} algorithm. At time $t=1$, OMWU outputs the uniform distribution $\vec{x}^{(1)}= \frac{1}{n}\vec{1}\in\Delta^{n}$. Then, at each time $t \ge 1$, it computes the next iterate $\vec{x}^{(t+1)}$
%incorporating the feedback loss vector $\vec{\ell}_i^{(t)}$
according to the update rule
\begin{equation}
    \label{eq:omwu}
    \tag{OMWU}
    \vec{x}^{(t+1)}[j] = \frac{\exp \{ - \eta (2 \vec{\ell}^{(t)}[j] - \vec{\ell}^{(t-1)}[j]) \} }{\sum_{k \in \range{n}}  \exp \{ - \eta (2 \vec{\ell}^{(t)}[k] - \vec{\ell}^{(t-1)}[k]) \}\, \vec{x}^{(t)}[k]}\, \vec{x}^{(t)}[j] ,
\end{equation}
where $\eta > 0$ is a \emph{learning rate} parameter and $\vec{\ell}^{(t)}$ is the feedback loss vector observed at time $t$ (we conventionally let $\vec{\ell}^{(0)} = \vec{0}$).

%In a recent breakthrough, Daskalakis, Fishelson, and Golowich established the following result \cite[Theorem 3.1]{Daskalakis21:Near}:

%\begin{theorem}[\citep{Daskalakis21:Near}]
%If all agent play according to \eqref{eq:omwu} with a sufficiently %small learning rate $\eta = O(1/(m \log^4 T))$, the external regret of %every player $i$ grows as $\reg^T = O(m \log n \log^4 T)$.
%\end{theorem}

\subsection{The Markov Chain Tree Theorem}

Given an $n$-state ergodic (\emph{i.e.}, aperiodic and irreducible) Markov chain with (row-stochastic) transition matrix $\mat{Q}$, the classic Markov chain tree theorem provides a closed-form solution for its \emph{stationary distribution}, that is, the unique distribution $\vec{\pi} \in \Delta^n$ such that $\vec{\pi}^\top \mat{Q} = \vec{\pi}^\top$. The result requires the notion of a \emph{directed rooted tree} (a.k.a. \emph{arborescence}), recalled next.

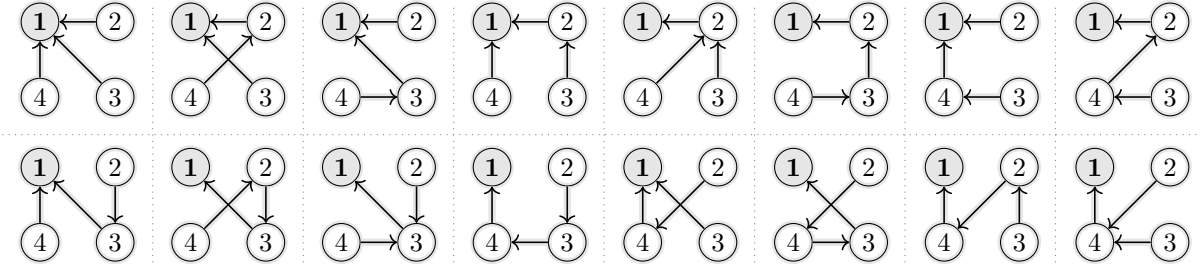
\begin{figure}[t]
    \centering\begin{tikzpicture}
    \tikzstyle{nd}=[draw,circle,inner sep=0,minimum size=5mm];
    \tikzstyle{ed}=[black,semithick];
    \tikzstyle{oln}=[line width=0.8mm,gray!20];
    \tikzstyle{ole}=[line width=0.99mm,gray!20];
    \foreach[count=\i from 0] \a/\b/\c in {
      1/1/1,%
      1/1/2,%
      1/1/3,%
      1/2/1,%
      1/2/2,%
      1/2/3,%
      1/4/1,%
      1/4/2,%
      3/1/1,%
      3/1/2,%
      3/1/3,%
      3/4/1,%
      4/1/1,%
      4/1/3,%
      4/2/1,%
      4/4/1%
    }{
        \pgfmathdiv{\i}{8}\let\row=\pgfmathresult
        \pgfmathmod{\i}{8}\let\col=\pgfmathresult
        \begin{scope}[xshift=57*\col,yshift=-55*\row]
            \draw[oln] (0, 0) circle (2.5mm);
            \draw[oln] (1, 0) circle (2.5mm);
            \draw[oln] (0,-1) circle (2.5mm);
            \draw[oln] (1,-1) circle (2.5mm);
            \node[nd,fill=gray!20] (n1) at (0, 0) {\small $\mathbf{1}$};
            \node[nd] (n2) at (1, 0) {\small $2$};
            \node[nd] (n3) at (1,-1) {\small $3$};
            \node[nd] (n4) at (0,-1) {\small $4$};
            \draw[ole] (n2) -- (n\a);
            \draw[ole] (n3) -- (n\b);
            \draw[ole] (n4) -- (n\c);
            \draw[ed,->] (n2) -- (n\a);
            \draw[ed,->] (n3) -- (n\b);
            \draw[ed,->] (n4) -- (n\c);
        \end{scope}
    }
    \draw[thin,dotted,gray] (-.5, -1.5) -- ++(16,0);
    \foreach \i in {1,...,7}{
        \draw[thin,dotted,gray] (2*\i-.5,-3.2) -- ++(0,3.45);
    }
\end{tikzpicture}
    \caption{Set $\treeset[1]^4$ of all directed trees rooted at node $1$ in a graph with $4$ nodes.}
    \label{fig:directed_trees}
\end{figure}

\begin{definition}
    \label{definition:directed_tree}
    Let $\tree = (V, E)$, with $V = \range{n}$, be an $n$-node directed graph. $\tree$ is a \emph{directed tree rooted at $j \in V$} if (i) it contains no cycle (including self-loops); (ii) every node $V \setminus \{j\}$ has \emph{exactly one} outgoing edge; and (iii) node $j$ has \emph{no} outgoing edges.
% \end{definition}
% \begin{definition}
    We denote the set of all $n$-node directed trees rooted at $j\in\range{n}$ with the symbol $\treeset[j]^n$. Furthermore, we let $\treeset^n \defeq \treeset[1]^n \cup \dots \cup \treeset[n]^n$.% the (disjoint) union of all sets $\treeset[j]$.
\end{definition}
\noindent It is well-known that $|\treeset[j]^n| = n^{n-2}$ for all $j \in \range{n}$, and consequently $|\treeset^n| = n^{n-1}$~\citep{Cayley89:Theorem}. An example for $n = 4$ is given in \cref{fig:directed_trees}.
%For simplicity, we will omit the superscript $n$ as it will be clear from the context.
In this context, the Markov chain tree theorem asserts that the vector $(\Sigma_1,\dots,\Sigma_n)$ of quantities
\begin{equation}
    \label{eq:Sigma}
    \Sigma_j \coloneqq \sum_{\tree \in \treeset[j]} \prod_{(a, b) \in E(\tree)} \mat{Q}[a,b]\qquad\qquad\qquad (j\in\range{n})
\end{equation}
is proportional to the stationary distribution. Specifically, we have the following.
\begin{theorem}[Markov chain tree theorem]
    \label{theorem:tree_theorem}
The stationary distribution $\vec{\pi}$ of an $n$-state ergodic Markov chain satisfies $\vec{\pi}[j] = {\Sigma_j}/{\Sigma}$
for all $j \in \range{n}$, where $\Sigma \coloneqq \Sigma_1 + \dots + \Sigma_n$.
\end{theorem}
For a proof of the Markov chain tree theorem, we refer the interested reader to the works of \citet{Anatharam89:Proof,Kruckman10:Elementary}.

\section{Near-Optimal No-Internal-Regret Dynamics for CE}
\label{section:main}

In this section we establish our main result regarding algorithm $\slomwu$, namely \Cref{theorem:main summary}. In particular, we start in \cref{section:revisit} by formalizing one of our key insights: $\slomwu$ can be equivalently thought of as a certain linear transformation of the output of an external regret minimizer operating over the combinatorial space of directed trees; this equivalence is formalized in \cref{theorem:equivalence}. Next, in \cref{section:external_regret} we leverage this connection to bound the internal regret of $\slomwu$ using an extension of the techniques developed in \citep{Daskalakis21:Near}.

\subsection{Equivalence Result}
\label{section:revisit}

Before we proceed with the statement and proof of \cref{theorem:equivalence}, we first summarize $\slomwu$ in \cref{algo:irm}; a more comprehensive overview of this construction is included in \Cref{appendix:SL}. In addition, in \cref{algo:arbo based irm} we present an external regret minimizer over the space of all directed trees (a.k.a. arborescences). Both \cref{algo:irm,algo:arbo based irm} use the symbol $\vec{x}^{(t)}$ to denote the output strategies; this choice is justified by the following theorem (see also \cref{fig:equivalence}).

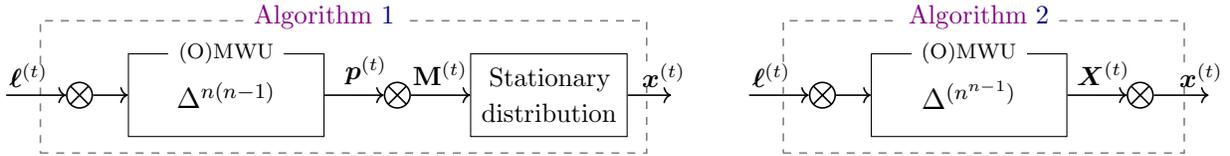
\begin{figure}[ht]
    \centering\scalebox{1}{\begin{tikzpicture}[yscale=1.1,xscale=1.3]
    \draw[semithick,dashed,gray] (0.1,0.3) rectangle (6.3,1.9) node[fitting node] (outer) {};
    \draw (1,0.5) rectangle (3,1.5) node[fitting node] (rm) {};
    \draw (4.5,0.5) rectangle (6.1,1.5) node[fitting node] (fp) {};
    \node[ox] (t1) at (0.5,1) {};
    \node[ox] (t2) at (3.75,1) {};
    \node at (rm.center) {$\Delta^{n (n-1)}$};
    \node[text width=2.5cm,align=center,inner sep=0pt] at (fp.center) {\small Stationary\\ distribution};
    \draw[semithick,->] (-.25,1) node[above,xshift=3mm,yshift=1mm,inner ysep=0,fill=white]{\small$\vec{\ell}^{(t)}$} -- (t1);
    \draw[semithick,->] (t1) -- (rm);
    \draw[semithick,->] (rm) -- (t2) node[above left] {\small$\vec{p}^{(t)}$};
    \draw[semithick,->] (t2) -- (fp.west) node[above,xshift=-4mm] {\small$\mat{M}^{(t)}$};
    \draw[semithick,->] (fp) -- (6.55,1) node[above,inner ysep=0,fill=white,yshift=1mm,xshift=-1mm] {\small$\vec{x}^{(t)}$};
    \node[fill=white,yshift=0mm] at (rm.north) {\fontsize{8}{8}\selectfont (O)MWU};
    \node[fill=white,inner ysep=0pt] at (3.0, 1.95) {\small\color{violet}\cref{algo:irm}};
\end{tikzpicture}}%
    \hfill%
    \scalebox{1}{\begin{tikzpicture}[yscale=1.1,xscale=1.3]
    \draw[semithick,dashed,gray] (0.1,0.3) rectangle (4.2,1.9) node[fitting node] (outer) {};
    \draw (1,0.5) rectangle (3,1.5) node[fitting node] (rm) {};
    \node[ox] (t1) at (0.5,1) {};
    \node[ox] (t2) at (3.75,1) {};
    \node at (rm.center) {$\Delta^{(n^{n-1})}$};
    \draw[semithick,->] (-.25,1) node[above,xshift=3mm,yshift=1mm,inner ysep=0,fill=white]{\small$\vec{\ell}^{(t)}$} -- (t1);
    \draw[semithick,->] (t1) -- (rm);
    \draw[semithick,->] (rm) -- (t2) node[above left] {\small$\vec{X}^{(t)}$};
    \draw[semithick,->] (t2) -- (4.45,1) node[above,inner ysep=0,fill=white,yshift=1mm,xshift=-1mm] {\small$\vec{x}^{(t)}$};
    \node[fill=white,yshift=0mm] at (rm.north) {\fontsize{8}{8}\selectfont (O)MWU};
    \node[fill=white,inner ysep=0pt] at (2.1, 1.95) {\small\color{violet}\cref{algo:arbo based irm}};
\end{tikzpicture}}
    \caption{A schematic illustration of the equivalence result of \cref{theorem:equivalence}; $\otimes$ in the figure represents a linear transformation.}
    \label{fig:equivalence}
\end{figure}

\begin{figure}[H]%
\begin{minipage}[t]{.49\linewidth}%
  \begin{algorithm}[H]
  \SetAlgoLined
    \SetInd{2.3mm}{2.3mm}
    \DontPrintSemicolon
    \KwData{$\mathcal{R}_{\Delta}$: OMWU algorithm for $\Delta^{n(n-1)}$ with learning rate $\eta > 0$}
    \BlankLine{}
    \Fn{$\nextstr()$}{
        $\vec{p}^{(t)} \gets \mathcal{R}_{\Delta}.\nextstr()$\\
        $\mat{M}^{(t)} \gets \displaystyle\sum_{j \neq k \in \range{n}} \vec{p}^{(t)}[\edge{j}{k}]\, \mat{E}^\top_{j \to k}$ \label{line:M}\\
        \textbf{return} $\vec{x}^{(t)}\gets\fp(\mat{M}^{(t)})$%
        ~\makeatletter\footnotemark\makeatother\hspace*{-1cm}\;
    }
    \HlineWd{}
    \Fn{$\obsut(\vec{\ell}^{(t)})$}{
        $\vec{L}^{(t)} \gets \vec{0} \in \R^{n(n-1)}$\;
        \For{$j \neq k \in \range{n}$}{
            $\vec{L}^{(t)}[\edge{j}{k}] \gets \vec{x}^{(t)}[j] ( \vec{\ell}^{(t)}[k] - \vec{\ell}^{(t)}[j]) $\label{line:L irm}\;
        }
        $\mathcal{R}_{\Delta}.\obsut(\vec{L}^{(t)})$
    }
    \caption{\citet{Stoltz05:Internal}}
    \label{algo:irm}
  \end{algorithm}
\end{minipage}%
\hfill%
\begin{minipage}[t]{.49\linewidth}%
  \begin{algorithm}[H]
    \SetInd{2.3mm}{2.3mm}
    \DontPrintSemicolon
    \KwData{$\mathcal{R}_{\Delta}$: OMWU algorithm for $\Delta^{(n^{n-1})}$ with learning rate $\eta > 0$}
    \BlankLine{}
    \Fn{$\nextstr()$}{
        $\vec{X}^{(t)} \gets \mathcal{R}_{\Delta}.\nextstr()$\;\vspace{.5mm}
        \textbf{return} $\vec{x}^{(t)} \gets \mleft(\sum_{\tree \in \treeset[j]^n} \vec{X}^{(t)}[\tree]\mright)_{j=1}^n$ \;\label{line:xt arbo}
    }
    \HlineWd{}\vspace{.4mm}
    \Fn{$\obsut(\vec{\ell}^{(t)})$}{\vspace{.5mm}
        $\bigL^{(t)} = \vec{0} \in \R^{|\treeset^n|} = \R^{(n^{n-1})}$\;
        \For{$\tree \in \treeset^n$}{
            \makebox[3cm][l]{$\bigL^{(t)}[\tree] \leftarrow\displaystyle \!\!\!\!\!\sum_{(j, k)\in E(\tree)}\!\!\!\! \vec{x}^{(t)}[j] ( \vec{\ell}^{(t)}[k] - \vec{\ell}^{(t)}[j])$}\label{line:L arbo}\;
        }\vspace{.5mm}
        $\mathcal{R}_{\Delta}.\obsut (\bigL^{(t)})$\;
        \vspace{.3mm}
    }
    \caption{\makebox[3cm][l]{Arborescence-based dynamics\phantom{[}}}
    \label{algo:arbo based irm}
  \end{algorithm}
\end{minipage}
\end{figure}
\footnotetext{The matrix $\mat{M}^{(t)}$ has strictly positive entries at all times $t$, since the iterates $\vec{p}^{(t)}$ produced by \eqref{eq:omwu} lie in the relative interior of the simplex. So, each $\mat{M}^{(t)}$ admits a \emph{unique} fixed point (stationary distribution).}

\begin{theorem}
    \label{theorem:equivalence}
    For any learning rate $\eta > 0$, \cref{algo:irm,algo:arbo based irm} produce the same strategies $\vec{x}^{(1)}, \dots, \vec{x}^{(T)}\in\Delta^n$, assuming that they observe the same sequence of losses $\vec{\ell}^{(1)}, \dots, \vec{\ell}^{(T)} \in \R^n$.
\end{theorem}
\begin{proof}
We will inductively show that the following property holds.

\begin{property}\label{ppt:same strategies}
    At all times $t \ge 1$, the following conditions hold:
    \begin{enumerate}[nosep]
        \item there exists $N^{(t)} \in \R$ such that
        \begin{equation}
            \label{eq:propo}
            \prod_{(a, b) \in E(\tree)} \vec{p}^{(t)}[\edge{a}{b}] = N^{(t)} \vec{X}^{(t)}[\tree] \qquad \forall\,\tree \in \treeset^n,
        \end{equation}
        where $\vec{p}^{(t)}$ and $\vec{X}^{(t)}$ are as defined in \cref{algo:irm} and \cref{algo:arbo based irm}, respectively; and \label{item:N-t}
        \item the strategies produced by \cref{algo:irm,algo:arbo based irm} are the same. \label{item:goal}
    \end{enumerate}
\end{property}

Condition~\ref{item:goal} of \cref{ppt:same strategies} immediately implies the statement. First, we establish the base case $t = 1$. The first iterate of OMWU is always the uniform distribution; so, $\vec{p}^{(1)} = \nicefrac{1}{n(n-1)}\,\vec{1}\in \Delta^{n(n-1)}$. Hence, the matrix $\mat{M}^{(1)}$ (\cref{line:M} of \cref{algo:irm}) has entries $\mat{M}[j, k] = \nicefrac{1}{n(n-1)}$ for all $j \neq k \in \range{n}$. Correspondingly, the (unique) stationary distribution of $\mat{M}^{(1)}$ is the uniform distribution $\vec{x}^{(1)} = \frac{1}{n}\,\vec{1} \in \Delta^n$. We now verify that the same iterate is produced by \cref{algo:arbo based irm}. Since $\vec{X}^{(1)}$ is computed using OMWU, $\vec{X}^{(1)} \in \Delta^{n^{n-1}}$ is the uniform distribution $\vec{X}^{(1)}[\tree] = \nicefrac{1}{n^{n-1}}$. So, using the fact that $|\treeset[j]^n| = n^{n-2}$ for all $j \in \range{n}$, each coordinate of the output strategy of \cref{algo:arbo based irm} is equal to ${n^{n-2}}/{n^{n-1}} = \nicefrac{1}{n}$, establishing the base case for Condition~\ref{item:goal} of \cref{ppt:same strategies}. Furthermore, since $\vec{p}^{(1)}$ and $\vec{X}^{(1)}$ are the uniform distributions over $\Delta^{n(n-1)}$ and $\Delta^{(n^{n-1})}$ respectively, Condition~\ref{item:N-t} follows directly from the fact that any directed tree $\tree\in\treeset^n$ has exactly $n-1$ edges.

Next, we prove the inductive step. Assume that \cref{ppt:same strategies} holds at times $\tau = 1,\dots,t$, for some $t \geq 1$. We will prove that it will hold at time $t + 1$ as well. Since $\vec{p}^{(t+1)} \in \Delta^{n(n-1)}$ in \cref{algo:irm} is updated using \eqref{eq:omwu}, for all $j \neq k \in \range{n}$
\begin{equation}
    \label{eq:p-omwu}
    \vec{p}^{(t+1)}[\edge{j}{k}] = \frac{\exp\{- \eta (2 \vec{L}^{(t)}[\edge{j}{k}] - \vec{L}^{(t-1)}[\edge{j}{k}] )\}}{\sum_{a \neq b} \exp\{ - \eta (2 \vec{L}^{(t)}[\edge{a}{b}] - \vec{L}^{(t-1)}[\edge{a}{b}])\} \vec{p}^{(t)}[\edge{a}{b}]} \vec{p}^{(t)}[\edge{j}{k}],
\end{equation}
where the loss vector
\begin{equation}\label{eq:L irm}
    \vec{L}^{(t)}[\edge{j}{k}] \coloneqq \vec{x}^{(t)}[j] (\vec{\ell}^{(t)}[k] - \vec{\ell}^{(t)}[j])
\end{equation}
is as defined on \cref{line:L irm} of \cref{algo:irm}. For convenience, we will denote the denominator in \cref{eq:p-omwu} as $S^{(t+1)}$. Similarly, since the strategies $\vec{X}^{(t)}\in\Delta^{(n^{n-1})}$ of \cref{algo:arbo based irm} are updated using \eqref{eq:omwu} with the \emph{same} learning rate $\eta > 0$, we have
\begin{equation}
    \label{eq:X-omwu}
    \vec{X}^{(t+1)}[\tree] = \frac{\exp\{ - \eta (2 \bigL^{(t)}[\tree] - \bigL^{(t-1)}[\tree])\}}{\sum_{\tree' \in \treeset^n} \exp\{ - \eta (2 \bigL^{(t)}[\tree'] - \bigL^{(t-1)}[\tree']\})\, \vec{X}^{(t)}[\tree']} \vec{X}^{(t)}[\tree] \qquad \forall\,\tree\in\treeset^n,
\end{equation}
where $\bigL^{(t)}$ is defined on \cref{line:L arbo} of \cref{algo:arbo based irm} as
\begin{equation}
    \label{eq:calL}
\bigL^{(t)}[\tree] \defeq \sum_{(a, b) \in E(\tree)} \vec{x}^{(t)}[a] (\vec{\ell}^{(t)}[b] - \vec{\ell}^{(t)}[a]) = \sum_{(a, b) \in E(\tree)} \vec{L}^{(t)}[\edge{a}{b}] \qquad \forall\,\tree\in\treeset^n.
\end{equation}
(Note that in \eqref{eq:L irm} and \eqref{eq:calL} we implicitly used the inductive hypothesis that the strategies of the two algorithms coincide at iterate $t$, as well as the assumption that the sequence of losses $\vec{\ell}^{(t)}$ observed by \cref{algo:irm,algo:arbo based irm} is the same.) To simplify the notation, we will use $\mathcal{S}^{(t+1)}$ to represent the denominator in \eqref{eq:X-omwu}. Next, we observe that for any $\tree \in \treeset^n$
\begin{align}
    \prod_{(a, b) \in E(\tree)}\!\!\!\!\! \vec{p}^{(t+1)}[\edge{a}{b}] &= \mleft( \frac{1}{S^{(t+1)}} \mright)^{n-1} \!\!\!\!\!\prod_{(a, b) \in E(\tree)}\!\!\!\! \exp \mleft\{ -\eta (2 \vec{L}^{(t)}[\edge{a}{b}] - \vec{L}^{(t-1)}[\edge{a}{b}]) \mright\} \vec{p}^{(t)}[\edge{a}{b}] \label{eq:align-p-omwu} \\
    &= \mleft( \frac{1}{S^{(t+1)}} \mright)^{n-1}\!\!\exp\mleft\{ - \eta\!\! \sum_{\!\!\!\!{(a,b) \in E(\tree)}\!\!\!\!} \!\!\!2\vec{L}^{(t)}[\edge{a}{b}] - \vec{L}^{(t-1)}[\edge{a}{b}] \! \mright\}\! \prod_{(a, b) \in E(\tree)}\!\!\! \vec{p}^{(t)}[\edge{a}{b}]  \nonumber\\
    &= \mleft( \frac{1}{S^{(t+1)}} \mright)^{n-1} \exp \mleft\{ -\eta \mleft( 2 \bigL^{(t)}[\tree] - \bigL^{(t-1)}[\tree] \mright) \mright\} \! \prod_{(a, b) \in E(\tree)}\! \vec{p}^{(t)}[\edge{a}{b}], \label{eq:align-calL} 
\end{align}
where \eqref{eq:align-p-omwu} uses~\eqref{eq:p-omwu} and the fact that every directed tree $\tree$ has exactly $n-1$ edges, the second equality uses the multiplicative properties of the exponential, and \eqref{eq:align-calL} uses the definition of $\bigL^{(t)}$ given in \eqref{eq:calL}. From the inductive hypothesis, Condition~\ref{item:N-t} of \cref{ppt:same strategies} holds at time $t$; so, continuing \eqref{eq:align-calL},
\begin{align}
    \prod_{(a, b) \in E(\tree)} \vec{p}^{(t+1)}[\edge{a}{b}] %&= \mleft( \frac{1}{S^{(t+1)}} \mright)^{n-1} \exp \mleft\{ -\eta \mleft( 2 \bigL^{(t)}[\tree] - \bigL^{(t-1)}[\tree] \mright) \mright\} \prod_{(a, b) \in E(\tree)} \vec{p}^{(t)}[\edge{a}{b}] \notag \\
    &= \mleft( \frac{1}{S^{(t+1)}} \mright)^{n-1} \exp \mleft\{ -\eta \mleft( 2 \bigL^{(t)}[\tree] - \bigL^{(t-1)}[\tree] \mright) \mright\} N^{(t)} \vec{X}^{(t)}[\tree] \nonumber\\% \label{eq:align-ind} \\
    &= \mleft( \frac{1}{S^{(t+1)}} \mright)^{n-1} \mathcal{S}^{(t+1)} N^{(t)} \vec{X}^{(t+1)}[\tree], \label{eq:align-X-p}
\end{align}
where~\eqref{eq:align-X-p} follows from~\eqref{eq:X-omwu} and the definition of $\mathcal{S}^{(t+1)}$. As a result, we have shown that Condition~\ref{item:N-t} of \cref{ppt:same strategies} holds at time $t+1$ for the parameter
\begin{equation}
    \label{eq:normalization}
    N^{(t+1)} \defeq \mleft( \frac{1}{S^{(t+1)}} \mright)^{n-1} \mathcal{S}^{(t+1)} N^{(t)}.
\end{equation}
We now show that Condition~\ref{item:goal} of \cref{ppt:same strategies} holds at time $t+1$ as well. To do so, we analyze the (unique) stationary distribution of the matrix $\mat{M}^{(t+1)}$ defined on \cref{line:M} of \cref{algo:irm}. First, we start with the following simple observation.
\begin{claim}\label{cl:M to p}
For any $j \neq k \in \range{n}$, 
%\begin{equation}\label{eq:M to p}
$    \mat{M}^{(t+1)}[j, k] = \vec{p}^{(t+1)}[\edge{j}{k}]. $
%\end{equation}
\end{claim}
\begin{proof}
For any $j \neq k\in\range{n}$ the unique non-zero non-diagonal entry of the matrix $\mat{E}_{j \to k}$ appears as $\mat{E}_{j \to k}[k, j] = 1$ (recall \eqref{eq:phi jk}). Thus, the claim follows directly by the definition of the matrix $\mat{M}^{(t+1)}$ in \cref{line:M} of \cref{algo:irm}.
\end{proof}
In light of~\cref{cl:M to p}, the $j$-th coordinate of the fixed point $\vec{x}^{(t+1)}$ of $\mat{M}^{(t+1)}$ can be expressed using the Markov chain tree theorem (\cref{theorem:tree_theorem}) as
\begin{align}\label{eq:ugly}
    \vec{x}^{(t+1)}[j] &= \frac{\displaystyle\sum_{\tree\in\treeset[j]^n}\prod_{(a,b)\in E(\tree)} \mat{M}^{(t+1)}[a,b]}{\displaystyle\sum_{j=1}^n\sum_{\tree\in\treeset[j]^n}\prod_{(a,b)\in E(\tree)} \mat{M}^{(t+1)}[a,b]} = \frac{\displaystyle\sum_{\tree\in\treeset[j]}\prod_{(a,b)\in E(\tree)} \vec{p}^{(t+1)}[\edge{a}{b}]}{\displaystyle\sum_{j=1}^n\sum_{\tree\in\treeset[j]^n}\prod_{(a,b)\in E(\tree)} \vec{p}^{(t+1)}[\edge{a}{b}]}.
\end{align}
% where the second equality follows from~\cref{cl:M to p}.
Using~\eqref{eq:align-X-p} together with the fact that $\vec{X}^{(t+1)}\in \Delta^{(n^{n-1})}$, and therefore $\sum_{\tree\in\treeset^n}\vec{X}^{(t+1)}[\tree] = 1$, the denominator of~\eqref{eq:ugly} satisfies
\begin{equation}
    \label{eq:Sigma-notation}
    \sum_{j=1}^n \sum_{\tree \in \treeset[j]^n} \prod_{(a, b) \in E(\tree)} \vec{p}^{(t+1)}[\edge{a}{b}] =  \sum_{\tree\in\treeset^n} N^{(t + 1)} \vec{X}^{(t+1)}[\tree] = N^{(t+1)}.
\end{equation}
Similarly, using \cref{eq:align-X-p} together with~\eqref{eq:normalization}, the numerator of~\eqref{eq:ugly} can be expressed as
\begin{align}
    \sum_{\tree\in\treeset[j]^n}\prod_{(a,b)\in E(\tree)} \vec{p}^{(t+1)}[\edge{a}{b}] 
    &= N^{(t+1)} \sum_{\tree \in \treeset[j]^n} \vec{X}^{(t+1)}[\tree]. \label{eq:align-inductive-X-p}
\end{align}
Plugging~\eqref{eq:Sigma-notation} and~\eqref{eq:align-inductive-X-p} into~\eqref{eq:ugly} we conclude that
$
    \vec{x}^{(t+1)}[j] = \sum_{\tree\in\treeset[j]^n} \vec{X}^{(t+1)}[\tree],
$
which is exactly the $j$-th coordinate of the iterate produced by \cref{algo:arbo based irm} at time $t + 1$ (\cref{line:xt arbo}). Thus, the strategies of \cref{algo:irm,algo:arbo based irm} at time $t + 1$ are the same, completing the inductive proof.
\end{proof}

% \begin{remark}
% The equivalence established in \cref{theorem:equivalence} applies more generally for the \emph{predictive} version of MWU, no matter the sequence of predictions. In particular, it applies to 
% \end{remark}

%The importance of the equivalence established in \cref{theorem:equivalence} is that we can essentially replace in the analysis the fixed point operation by (a linear transformation of) the output of the regret minimizer over $\Delta^{n^{n-1}}$. 
%Armed with this connection, we can now follow the approach in \citep{Daskalakis21:Near} to bound the external regret of the regret minimizer $\mathcal{R}_{\Delta}$ employed in \cref{algo:irm}, 

\subsection{Bounding the Internal Regret}
\label{section:external_regret}

Here we explain how to leverage the techniques in \citep{Daskalakis21:Near} to bound the external regret of $\mathcal{R}_{\Delta}$ employed in \Cref{algo:irm}, which also bounds the internal regret of $\slomwu$ (see \Cref{appendix:SL}). In particular, the crux in the analysis lies in showing that the sequence of observed losses of $\mathcal{R}_{\Delta}$ exhibits \emph{higher-order smoothness}. Let us first recall the notion of finite differences.

\begin{definition}
    \label{definition:fin-dif}
Consider a sequence of vectors $\vec{z} = (\vec{z}^{(1)}, \dots, \vec{z}^{(T)})$. For an integer $h \geq 0$, the $h$\emph{-order finite difference} for the sequence $\vec{z}$, denoted by $D_h \vec{z}$, is the sequence
\[
    D_h \vec{z} \coloneqq ((D_h \vec{z})^{(1)}, \dots, (D_h \vec{z})^{(T - h)})
\]
defined recursively as $(D_0 \vec{z})^{(t)} \coloneqq \vec{z}^{(t)}$, for $1 \leq t \leq T$, and
\begin{equation}
    \label{eq:finite_diff}
    (D_h \vec{z})^{(t)} \coloneqq (D_{h-1} \vec{z})^{(t+1)} - (D_{h-1} \vec{z})^{(t)},
\end{equation}
for $h \geq 1$, and $1 \leq t \leq T - h$.
\end{definition}
To establish higher-order smoothness, we use \Cref{theorem:equivalence} to ``lift'' the analysis to the regret minimizer over the arborescences (\Cref{algo:arbo based irm}). Then, we leverage the particular structure of the losses in $\slomwu$ to adapt the argument in \citep{Daskalakis21:Near}, leading to the following guarantee.

\begin{restatable}[]{lemma}{highsmooth}
\label{lemma:high-order-smooth}
Consider a parameter $\alpha \leq 1/(H+3)$. If all players employ $\slomwu$ with learning rate $\eta \leq \frac{\alpha}{36 e^5 m}$, then for any player $i \in \range{m}$, $0 \leq h \leq H$ and $t \in \range{T - h}$ it holds that 
\begin{equation*}
    \| (D_h \vec{L}_i)^{(t)} \|_{\infty} \leq \alpha^h h^{3h + 1}.
\end{equation*}
\end{restatable}
Recall that $\vec{L}_i^{(t)}$ is the loss observed by algorithm $\mathcal{R}_{\Delta}$ at time $t$ (see \Cref{algo:irm}). Armed with this crucial lemma, we are also able to extend the other technical ingredients used in \citep{Daskalakis21:Near}, and we combine them to conclude the proof of \Cref{theorem:main summary}, as we formally show in \Cref{appendix:proofs-3}.

\paragraph{Adversarial Bound.} The learning algorithm can also be slightly modified in order to guarantee robustness when faced against adversarial losses. Indeed, the following corollary implies near-optimal internal regret in both regimes. 

\begin{corollary}
    \label{corollary:adversarial}
There exists a learning algorithm such that, if employed by all players, it guarantees that the internal regret of each player $i \in \range{m}$ is bounded by $O(m \log n_i \log^4 T)$. Moreover, under adversarial losses the algorithm ensures internal regret bounded by $O(m \log n_i \log^4 T + \sqrt{\log n_i T})$.
\end{corollary}
The idea is to use an adaptive choice of learning rate depending on whether the bound predicted in \eqref{eq:check} has been violated in some repetition of the game, analogously to \cite[Corollary D.1]{Daskalakis21:Near}.
\section{Analysis of the Blum-Mansour Algorithm}

In this section we give an overview of our analysis for the no-swap-regret algorithm of \citeauthor{Blum07:From} \citep{Blum07:From} (\BM). Unlike the no-internal-regret algorithm of \citet{Stoltz05:Internal}, \BM maintains $n$ independent no-external-regret algorithms $\mathcal{R}_{\Delta,1}, \ldots, \mathcal{R}_{\Delta,n}$, operating over $\Delta^{n}$. %For this reason, our previous approach appears to be no longer applicable.
%
% Instead, w
We develop robust techniques that delve into the inner-workings of the higher-order smoothness argument in \cite{Daskalakis21:Near}. In particular, we substantially generalize their approach by demonstrating how higher-order smoothness bounds can be established even under the additional complexity of fixed point operations. %Our new technical lemmas are included in \Cref{appendix:proofs-BM}.
%
%In this way, we providing insight into the various ways to extend the techniques of \cite{Daskalakis21:Near} to more involved learning algorithms. 

First, to keep the exposition reasonably self-contained, let us briefly recall the \BM algorithm; a more detailed overview is included in \Cref{appendix:BM}.
At each time $t \geq 1$, every algorithm $\mathcal{R}_{\Delta, g}$ produces an iterate $\mat{Q}\^t[g,\rowdot] = (\mat{Q}\^t[g,1],\ldots,\mat{Q}\^t[g,n]) \in \Delta^{n}$, for all $g \in \range{n}$. Then, the algorithm computes a stationary distribution $\Delta^{n}\ni \vec{x}\^t = (\mat{Q}\^t)^\top \vec{x}\^t$ of the transition matrix $\mat{Q}\^t$. %, whose $g$-th row is the vector $\mat{Q}\^t[g,\rowdot]$. 
Moreover, upon receiving a loss vector $\vec{\ell}\^t \in \R^{n}$, the \BM algorithm distributes the loss vector $\vec{x}\^t[g] \vec{\ell}\^t$ to each regret minimizer $\mathcal{R}_{\Delta, g}$ for $g \in \range{n}$. In what follows, we will be concerned with the particular case where each regret minimizer is set to OMWU. 

Our primary technical contribution in the analysis of $\bmomwu$ is to show that the losses observed by each individual regret minimizer exhibit high-order smoothness. Namely, we show the following lemma.

\begin{lemma}\label{lem:dh-bound}
  Fix a parameter $\alpha \in \left(0, \frac{1}{H+3} \right)$. 
  There exists a sufficiently large universal constant $C$ such that, if all players follow $\bmomwu$ with learning rate $\eta \leq \frac{1}{CH^2n_i^3}$ %\leq \frac{\alpha}{36 e^5 m}$
  , then for any player $i \in \range{m}$, integer $h$ satisfying $0 \leq h \leq H$, time step $t \in \range{T-h}$, and $g \in \range{n_i}$, it holds that
\begin{align}
\| \fds{h}{(\vec{x}_i[g] \cdot \vec{\ell}_i)}{t} \|_\infty \leq \alpha^h h^{3h+1}\nonumber.
\end{align}
\end{lemma}

%Lemma

% Gabri: I have commented this out and moved it higher up in the section before all the math starts
% Max: bet sounds good
%Our proof that \BM-\Opthedge exhibits higher-order smoothness does not simply use the high-order smoothness of \Opthedge as a black-box.  Rather, the proof delves into the inner-workings of \cite{Daskalakis21:Near} and generalizes the arguments, demonstrating how similar smoothness bounds can be achieved even in the face of the added fixed-point computation component of the algorithm.  In addition to shedding clarity on the regret performance of \BM-\Opthedge, this proof provides insight into the various ways to extend the techniques of \cite{Daskalakis21:Near} to more complicated learning algorithms.\\

%In contrast to the arguments presented in the previous section, our proof that \BM-\Opthedge exhibits higher-order smoothness does not simply use the high-order smoothness of \Opthedge as a black-box.\footnote{The previous section does not use it as a black box}  Rather, the proof delves into the inner-workings of \cite{Daskalakis21:Near} and generalizes the arguments, demonstrating how similar smoothness bounds can be achieved even in the face of the added fixed-point computation component of the algorithm.  In addition to shedding clarity on the regret performance of \BM-\Opthedge, this proof provides insight into the various ways to extend the techniques of \cite{Daskalakis21:Near} to more complicated learning algorithms.\\

At a high level, the proof of this higher-order smoothness lemma hinges on the cyclic relationship between the losses incurred by the players $\vec{\ell}_i\^{t}$, the iterates produced by the copies of the OMWU algorithm $\mat{Q}_i\^{t}$, and the final strategies $\vec{x}_i\^{t}$ output by \BM-OMWU.   
In particular, the iterates $\mat{Q}_i\^{t}[g,\rowdot]$ are determined based on the overall history of loss vectors $\sum_{t'<t}\vec{x}_i\^{t'}[g] \vec{\ell}_i\^{t'}$ weighted exponentially in terms of the \emph{softmax} function; the strategies $\vec{x}_i\^{t}$ are determined by applying the Markov Chain Tree Theorem (\Cref{theorem:tree_theorem}) to the stochastic matrix $\mat{Q}_i\^{t}$ formed by the previous iterates; and, the losses $\vec{\ell}_i\^{t}$ are determined as a function of the strategies of all the other players $\vec{x}_{-i}\^{t}$. %and the game matrix $\ML$.  
Therefore, using the ``Boundedness Chain Rule for Finite Differences'' (Lemma \ref{lemma:boundedness-chain-rule}), one of the main technical tools shown in \cite{Daskalakis21:Near}, we can demonstrate that bounds on the $h$-th order finite differences of $\vec{x}_i\^t$ and $\vec{\ell}_i\^t$ imply a bound on the $(h+1)$-th order finite differences of $\mat{Q}\^{t}$.\footnote{This is due to the fact that $D_{h+1} \mat{Q}\^{t}[g,\rowdot] = D_{h+1} \sum_{t'<t}\vec{x}_i\^{t'}[g] \vec{\ell}_i\^{t'} = D_h \vec{x}_i\^{t}[g] \vec{\ell}_i\^{t}$.}
%is derived in terms of a sum over all $\vec{x}\^{t'}[g] \vec{\ell}\^{t'}$. For finite differences, this is effectively an integral, and the $(h+1)$-th \todo{missing ``order finite difference'' here?} of the integral of a sequence is equivalent to the $h$-th derivative of that sequence.}
This, in turn, implies a bound on the $(h+1)$-th order finite differences of $\vec{x}_i\^t$, which then gives a bound on the $(h+1)$-th order finite differences of $\vec{\ell}_i\^t$, as long as the Taylor coefficients of the softmax function and the Markov Chain Tree Theorem are sufficiently well-behaved. While this is not quite the case, we make a slight modification to this cycle of implication, bounding the finite differences of $\log \mat{Q}_i\^{t}$ instead.
\begin{equation*}
    D_h \vec{x} \text{ bound},D_h\vec{\ell} \text{ bound} \Rightarrow D_{h+1}\paren{\log \mat{Q}} \text{ bound} \Rightarrow D_{h+1} \vec{x} \text{ bound} \Rightarrow D_{h+1}\vec{\ell}\text{ bound}.
\end{equation*}
This cycle of implication enables an inductive argument that proves Lemma \ref{lem:dh-bound} as long as the log of the softmax function and an exponential version of the Markov Chain Tree Theorem exhibit bounded Taylor coefficients. These bounds are proved in Lemma \ref{lem:log-softmax-ak-bound} and Lemma \ref{lem:mct-taylor} respectively. The proof of Lemma \ref{lem:log-softmax-ak-bound} follows an explicit, combinatorial framework similar to that presented in \cite{Daskalakis21:Near}. On the other hand, Lemma \ref{lem:mct-taylor} introduces a novel technique for proving Taylor coefficient bounds, applying the Cauchy Integral Formula. This approach is far more general, and could help establish the necessary preconditions of Boundedness Chain Rule Lemma in much broader settings wherein combinatorial approaches are insufficient.
\section{Conclusions and Open Problems}

In conclusion, we have extended the recent result of \citet{Daskalakis21:Near} from external to internal and swap regret. As a corollary, we obtained the first near-optimal---within the no-regret framework---rates of convergence for correlated equilibrium. To do so, we developed several new techniques that allowed us to establish higher-order smoothness for no-internal and no-swap learning dynamics. %\todo{Repeat the key ingredients/insights?}

Finally, we identify several possible avenues for future research related to our results.
\begin{itemize}[nosep]
    \item Although our internal-regret bounds are near-optimal in terms of the dependency on the number of actions $n_i$ of each player~$i$, for swap regret our bounds depend polynomially on $n_i$. While a polynomial dependence on $n_i$ is necessary in the adversarial setting \citep{Blum07:From,Ito20:A}, we are not aware of any lower bounds for the setting of smooth, predictable sequences of losses within which our paper operates.
    \item Can our results be extended beyond OMWU, for example to other instances of the general optimistic FTRL algorithm~\citep{Syrgkanis15:Fast}? %It seems that our equivalence theorem (\cref{theorem:equivalence}) crucially relies on the multiplicative structure of the update rule of OMWU, while it should also be noted that such results \todo{Not clear what this is referring to?} are currently not known even for external regret.
    \item Finally, our equivalence theorem (\cref{theorem:equivalence}) was only established with respect to the set $\Phi_i^\text{int}$ of transformations corresponding to internal regret. Is it possible to extend our results beyond such transformations (\emph{e.g.}, see \citep{Farina21:Simple}) via closed-form formulas for the associated fixed points, analogous to the Markov chain tree theorem? Exploring such connections further constitutes a promising direction for the future. 
\end{itemize}

\section*{Acknowledgements \& Funding}
We are grateful to the anonymous reviewers at STOC `22 for their many helpful comments. C.D.~is supported by NSF Awards CCF-1901292,  DMS-2022448 and  DMS-2134108, by a Simons Investigator Award, by the Simons Collaboration on the Theory of Algorithmic Fairness, by a DSTA grant, and by the DOE PhILMs project (No. DE-AC05-76RL01830). N.G.~is supported by a Fannie \& John Hertz Foundation Fellowship and an NSF Graduate Fellowship. T.S. is supported by NSF grants IIS-1718457, IIS-1901403, and CCF-1733556, and ARO award W911NF2010081.

\printbibliography
\clearpage

\appendix

\section{Phi-Regret Minimization}
\label{appendix:Phi}

In this section we give an overview of the no-internal-regret algorithm of \citet{Stoltz05:Internal} and the no-swap-regret algorithm of \citet{Blum07:From}. To do so, we will use the template developed by \citet{Gordon08:No} for performing Phi-regret minimization.

\subsection{From Internal Regret to External Regret} 
\label{appendix:SL}

In \Cref{section:prelim} we discussed the implications of having access to a no-internal-regret algorithm, namely \cref{theorem:convergence_correlated}, assuming that such objects indeed exist. In the sequel, this presumption will be justified by presenting an explicit construction that leads to such guarantees. In the sequel, we present the reduction from internal to external regret based on the template of \citet{Gordon08:No} for performing $\Phi$-regret minimization for a generic set $\Phi$. We remark that the subsequent construction is equivalent to the one appearing in the book of \citet[pp. 84]{Cesa-Bianchi06:Prediction}, which in turn is based on the work of~\citet{Stoltz05:Internal}.

In this setting, the construction of \citet{Gordon08:No} shows that $\Phi$-regret minimization can be achieved starting from the following two ingredients:
\begin{enumerate}
    \item An \emph{external} regret minimizer for the set of transformations $\Phi$; \label{item:i-phi} 
    \item Establishing an oracle that (efficiently) computes a fixed point for each $\phi \in \Phi$ (assuming that such a point exists). \label{item:ii-phi} 
\end{enumerate}

In this subsection we are interested in the case $\Phi = \intphi_i$ for each player $i\in\range{m}$. Let us focus first on item \ref{item:i-phi}. In particular, recalling the definition of the matrix $\mat{E}_{j \to k}$ associated with the linear transformation $\phi_{j \to k}$ given in \eqref{eq:phi jk}, for all $j \neq k \in \range{n_i}$, it is easy to see that one can construct a regret minimizer for the set $\intphi$ using a regret minimizer for the simplex $\Delta^{n_i(n_i-1)}$. This construction is summarized in \cref{algo:Phi-int}, and the formal guarantee is stated below:

%For any pair of actions $j\neq k \in \range{n}$, let $\mat{E}_{j\to k}$ be the matrix associated with the linear map . Since there are exactly $n(n-1)$ such linear maps, we can construct an external regret minimizer for set set $\Phi$ by tracking the regret cumulated by each of the individual transformations in $\Phi$. \todo{Add more details and citations. McMahan also mentioned regret minimization on polyhedral sets}

\begin{proposition}
If the regret minimizer $\mathcal{R}_{\Delta}$ for the set $\Delta^{n_i(n_i-1)}$ incurs regret at most $\reg_i^T$, then the regret minimizer $\mathcal{R}_{\Phi}$ for the set $\intphi_i$ in \cref{algo:Phi-int} has also regret bounded by $\reg_i^T$.
\end{proposition}

\begin{figure}[ht]
\begin{minipage}{.49\textwidth}
\begin{algorithm}[H]
    \SetAlgoLined
    \DontPrintSemicolon
    \KwIn{A regret minimizer $\mathcal{R}_{\Delta}$ for the simplex $\Delta^{n_i(n_i-1)}$}    
    \BlankLine
    \SetInd{2.3mm}{2.3mm}
    \Fn{$\nextstr()$}{
        $\vec{p}^{(t)} \leftarrow \mathcal{R}_{\Delta}.\nextstr()$\\
        \textbf{return} $ \sum_{j \neq k} \vec{p}^{(t)}[\edge{j}{k}] \, \mat{E}_{j\to k} $
    }
    \Hline{}
    \Fn{$\obsut(\mat{L}^{(t)})$}{
    \For{$j \neq k \in \range{n}$}{
        $\vec{L}^{(t)}[\edge{j}{k}] \leftarrow \langle \mat{L}^{(t)}, \mat{E}_{j\to k} \rangle$
    }
    $\mathcal{R}_{\Delta}.\obsut(\vec{L}^{(t)})$
    }
\caption{Regret minimizer for the set $\intphi_i$}
\label{algo:Phi-int}
\end{algorithm}
\end{minipage}\hfill
\begin{minipage}{.49\textwidth}
\begin{algorithm}[H]
    \SetAlgoLined
    \SetInd{2.3mm}{2.3mm}
    \DontPrintSemicolon
    \KwData{A regret minimizer $\mathcal{R}_{\Phi}$ for the set $ \intphi_i$, \emph{e.g.}, via \cref{algo:Phi-int}}    
        \BlankLine{}
        \Fn{$\nextstr()$}{
            $\phi^{(t)} \leftarrow \mathcal{R}_{\Phi}.\nextstr()$ \\
            $\vec{x}^{(t)} \leftarrow \fp(\phi^{(t)})$ \\
            \textbf{return} $\vec{x}^{(t)}$
        }
        \Hline{}
        \Fn{$\obsut(\vec{\ell}^{(t)})$}{
            Construct the loss $\mat{L}^{(t)} \leftarrow \vec{\ell}^{(t)} (\vec{x}^{(t)})^{\top}$ \\
            $\mathcal{R}_{\Phi}.\obsut(\mat{L}^{(t)})$
        }
    \caption{Internal Regret minimizer for the set $\Delta^{n_i}$}
    \label{algo:Gordon-int}
    \end{algorithm}
\end{minipage}
\end{figure}

Next, we turn our attention to item \ref{item:ii-phi}: establishing a fixed point oracle for each transformation $\phi \in \intphi$. To do so, we require the following key observation:

\begin{claim}
    \label{claim:phi-markov}
Each transformation $\phi \in \intphi_i$ represents the transition matrix of a Markov chain.
\end{claim}

\begin{proof}
Fix $j \neq k \in \range{n_i}$. It is immediate to see that each column of the matrix $\mat{E}_{j \to k}$, defined in \eqref{eq:phi jk}, corresponds to a point in the simplex $\Delta^{n_i}$. Thus, the same holds for any convex combination of such matrices, concluding the proof.
\end{proof}

This observation ensures that each $\phi \in \co \intphi_i$ admits a fixed point (\emph{e.g.}, as a corollary of the Perron-Frobenius theorem), associated with the stationary distribution of the corresponding Markov chain, and such a point can be computed in polynomial time. As a result, we have established both ingredients required in the template of \citet{Gordon08:No}, and the overall construction is presented in \cref{algo:Gordon-int}. Formally, the guarantee of this reduction is stated in the following proposition:

\begin{proposition}[From External to Internal Regret]
    \label{proposition:internal-external}
If the external regret of $\mathcal{R}_{\Phi}$ is bounded by $\reg_{\Phi_i}^T$, then the internal regret of \cref{algo:Gordon-int} is also bounded by $\reg_{\Phi_i}^T$.
\end{proposition}
\begin{proof}
The conclusion follows directly from \cite[Theorem 1]{Gordon08:No}.
\end{proof}

\subsection{From Swap Regret to External Regret}
\label{appendix:BM}

Next, we give an overview of the algorithm of \citet{Blum07:From}. In particular, in this case $\Phi_i = \Phi_i^\text{swap}$ is the set of \emph{all} linear transformations from $\Delta^{n_i}$ to $\Delta^{n_i}$. Following the construction in \citet{Gordon08:No}, let us first focus on item \ref{item:i-phi}: constructing an external regret minimizer for the set $\Phi_i^\text{swap}$. We will use the following well-known characterization.

\begin{fact}
\label{fact:trivial}
Any linear function $\Phi_i^\text{swap} \ni \phi : \Delta^{n_i} \to \Delta^{n_i}$ can be expressed in the form $\phi(\vec{x}) = \mat{Q}^{\top} \vec{x}$, where $\mat{Q}$ is a (row) stochastic matrix. 
\end{fact}

Moreover, the set of all column stochastic matrices can be expressed as a concatenation of $n_i$ independent columns, each of which is a vector in $\Delta^{n_i}$. As a result, an external regret minimizer for the set $\Phi_i^\text{swap}$ can be constructed by simply applying a ``regret circuit'' for the Cartesian product \citep{Farina19:Regret}. Next, focusing on item \ref{item:ii-phi}, computing the fixed point of each $\phi \in \Phi_i^\text{swap}$ is also direct since $\phi$ is associated with a Markov chain (\Cref{fact:trivial}). With these ingredients at hand we are ready to present the no-swap-regret algorithm of \citet{Blum07:From} in \Cref{algo:swap-BM}.

\begin{algorithm}[H]
    \SetAlgoLined
    \DontPrintSemicolon
    \KwIn{Regret minimizers $\mathcal{R}_{\Delta,1}, \dots, \mathcal{R}_{\Delta,n_i}$ for the simplex $\Delta^{n_i}$}    
    \BlankLine
    \SetInd{2.3mm}{2.3mm}
    \Fn{$\nextstr()$}{
        $\mat{Q}^{(t)} \gets \mat{0}\in\R^{n_i\times n_i}$\;
        \For{$g \in \range{n_i}$}{
        %$\vec{q}^{(t)}[g,\cdot] \leftarrow \mathcal{R}_{\Delta,g}.\nextstr()$\;
        % $\mat{Q}^{(t)}[g,\cdot] \leftarrow \vec{q}^{(t)}[g,\cdot]$\;
        $\mat{Q}^{(t)}[g,\cdot] \leftarrow \mathcal{R}_{\Delta,g}.\nextstr()$\;
        }
        $\vec{x}^{(t)}\gets\fp(\mat{Q}^{(t)})$\;
        \textbf{return} $\vec{x}^{(t)}$
    }
    \Hline{}
    \Fn{$\obsut(\vec{\ell}^{(t)})$}{
    \For{$g \in \range{n_i}$}{
        $\mathcal{R}_{\Delta, g}.\obsut(\vec{x}\^{t}[g]\, \vec{\ell}^{(t)})$
    }
    % \textbf{return} $\vec{x}^{(t)}\cdot \vec{\ell}\^{t}$
    }
\caption{\citet{Blum07:From}}
\label{algo:swap-BM}
\end{algorithm}

\section{Useful Technical Tools}

Most of the following technical ingredients were established in \citep{Daskalakis21:Near}, but we include them in order to keep the exposition reasonably self-contained. First of all, it will be useful to express $h$-order finite differences, as introduced in \cref{definition:fin-dif}, in the following form (see \cite[Remark 4.3]{Daskalakis21:Near}):

\begin{equation}
    \label{eq:binomial}
    (D_h \vec{z})^{(t)} = \sum_{s=0}^h \binom{h}{s} (-1)^{h - s} \vec{z}^{(t + s)}.
\end{equation}

Next, we state the ``boundedness chain rule'' for finite differences \citep{Daskalakis21:Near}. Specifically, let $\phi : \mathbb{R}^n \to \mathbb{R}$ be a real function analytic in a neighborhood of the origin. For real numbers $Q, R > 0$, we will say that $\phi$ is $(Q, R)$-\emph{bounded} if the Taylor series of $\phi$ with respect to the origin, denoted with $P_{\phi}(z_1, \dots, z_n) = \sum_{\gamma \in \Z_{\geq 0}^n } \alpha_{\gamma} \vec{z}^{\gamma}$ is such that 
\begin{equation}
    \sum_{\gamma \in \Z_{\geq 0}^n : |\gamma| = k} |\alpha_{\gamma}| \leq Q R^k,
\end{equation}
for any integer $k \geq 0$. The following result \cite[Lemma 4.5]{Daskalakis21:Near} is one of the central technical ingredients developed in \citep{Daskalakis21:Near}.

\begin{lemma}[\citep{Daskalakis21:Near}]
    \label{lemma:boundedness-chain-rule}
Consider a $(Q, R)$-bounded analytic function $\phi \in \R^n \to \R$ so that the radius of convergence of its power series with respect to the origin is greater than $\nu$, for some $\nu > 0$. Moreover, consider a sequence of vectors $(\vec{z}^{(1)}, \dots, \vec{z}^{(T)})$ such that $\| \vec{z}^{(t)}\|_{\infty} \leq \nu$ for all $t \in \range{T}$. Finally, suppose that for some parameter $\alpha \in (0, 1)$, and for each $0 \leq h' \leq h$ and $t \in \range{T - h'}$, it holds that\footnote{Here it is assumed that $0^0 = 1$.} $\| (D_{h'} \vec{z})^{(t)}\|_{\infty} \leq \frac{1}{B_1} \alpha^{h'} (h')^{B_0 h'}$, where $B_1 \geq 2e^2 R$ and $B_0 \geq 3$. Then, for all $t \in \range{T - h}$ it holds that 
\begin{equation*}
    \left\lvert (D_h(\phi \circ \vec{z}))^{(t)} \right\rvert \leq \frac{12 R Q e^2}{B_1} \alpha^h h^{B_0 h + 1}.
\end{equation*}
\end{lemma}
Note that in the statement of this lemma the notation $\circ$ represents the composition of functions. %Further, we remark that technically, the statement of \Cref{lemma:boundedness-chain-rule} in \citep{Daskalakis21:Near} is only stated for the special case $\nu = 1$. However, it is straightforward to see that the lemma holds for general $\nu$; indeed, the radius of convergence assumption and the $\ell_\infty$ norm bound $\|\vec{z}^{(t)}\|_\infty \leq \nu$ are only used to establish Eq.~(50) in \citep{Daskalakis21:Near}, which is immediate under the more general conditions stated here.

% We continue with \cite[Lemma C.4]{Daskalakis21:Near}; in particular, we say that a sequence of distributions $\vec{q}^{(1)}, \dots, \vec{q}^{(T)}$ is $\zeta$-multiplicatively close if $\max \mleft\{ \mleft\| \frac{\vec{q}^{(t)}}{\vec{q}^{(t+1)}} \mright\|_{\infty}, \mleft\| \frac{\vec{q}^{(t+1)}}{\vec{q}^{(t)}} \mright\|_{\infty} \mright\} \leq 1 + \zeta$, for all $1 \leq t \leq T - 1$, where we tacitly assume that each $\vec{q}^{(t)}$ assigns (strictly) positive probability at every entry, as ensured by \eqref{eq:omwu}, while the notation $\frac{\vec{q}^{(t)}}{\vec{q}^{(t+1)}}$ simply refers to the entry-wise division of the corresponding vectors. In the following statement, we also use the notation $\var_{\vec{q}} (\vec{z}) \coloneqq \sum_{j=1}^n \vec{q}(j) \mleft( \vec{z}(j) - \sum_{k=1}^n \vec{q}(k) \vec{z}(k) \mright)^2$.

\noindent We continue with \cite[Lemma C.4]{Daskalakis21:Near}. In the following statement, we use the notation $\var_{\vec{q}} (\vec{z}) \coloneqq \sum_{j=1}^n \vec{q}(j) \mleft( \vec{z}(j) - \sum_{k=1}^n \vec{q}(k) \vec{z}(k) \mright)^2$.

\begin{restatable}{lemma}{lgbld}
\label{lem:general-bound-l-dl}
{\normalfont (\citep{Daskalakis21:Near})}
    For any integers $n \geq 2$ and $T \geq 4$, we set $H := \lceil \log T \rceil$, $\alpha = 1/(4H)$, and $\alpha_0 = \frac{\sqrt{\alpha/8}}{H^3}$.  Suppose that $\vec{Z}\^1, \ldots, \vec{Z}\^T \in [0,1]^{n}$ and $\vec{P}\^1, \ldots, \vec{P}\^T \in \Delta^{n}$ satisfy the following
    \begin{enumerate}
        \item \label{it:gen-down} For each $0 \leq h \leq H$ and $1 \leq t \leq T-h$, it holds that $\left\| \fds{h}{\vec{Z}}{t} \right\|_\infty \leq H\cdot  \left(\alpha_0 H^{3}\right)^h$
        \item \label{it:gen-close} The sequence $\vec{P}\^1, \ldots, \vec{P}\^T$ is $\zeta$-consecutively close for some $\zeta \in [1/(2T), \alpha^4/8256]$. 
    \end{enumerate}
    Then,
    \begin{align*}
    \sum_{t=1}^T \Var{\vec{P}\^t}{\vec{Z}\^t - \vec{Z}\^{t-1}} \leq 2\alpha \sum_{t=1}^T \Var{\vec{P}\^t}{\vec{Z}\^{t-1}} + 165120(1+\zeta) H^5 + 2
    \end{align*}
\noindent where we say that a sequence $\vec{P}\^1, \ldots, \vec{P}\^T$ is $\zeta$-consecutively close iff for all $t$
\begin{equation*}
     \max \left\{\left\| \frac{\vec{P}\^t}{\vec{P}\^{t+1}}\right\|_\infty, \left\| \frac{\vec{P}\^{t+1}}{\vec{P}\^t}\right\|_\infty\right\} \leq 1+\zeta
\end{equation*}
\end{restatable}

\section{Analysis of the Stoltz-Lugosi Algorithm}
\label{appendix:proofs-3}

In this section we provide all of the technical ingredients required for the analysis of $\slomwu$, and subsequently for the proof of \Cref{theorem:main summary}. First, let us cast the refined bound under adversarial losses \cite[Lemma 4.1]{Daskalakis21:Near} in our setting.

\begin{lemma}[\citep{Daskalakis21:Near}]
    \label{lemma:RVU++}
Consider some player $i \in \range{m}$ employing $\slomwu$ with learning rate $\eta < 1/C$, where $C$ is a sufficiently large universal constant. Then, under any sequence of losses $\vec{L}_i^{(1)}, \dots, \vec{L}_i^{(T)}$, the external regret of $\mathcal{R}_{\Delta}$ can be bounded as
\begin{equation}
    \label{eq:RVU++}
    \reg_i^T \leq 2 \frac{\log n_i}{\eta} + \sum_{t=1}^T \mleft( \frac{\eta}{2} + C \eta^2 \mright) \var_{\vec{p}_i^{(t)}}\mleft( \vec{L}_i^{(t)} - \vec{L}_i^{(t-1)} \mright) - \sum_{t=1}^T \frac{(1 - C \eta) \eta}{2} \var_{\vec{p}_i^{(t)}} \mleft( \vec{L}_i^{(t-1)} \mright).
\end{equation}
\end{lemma}
Recall from \Cref{algo:irm} that $\vec{L}_i^{(t)}$ is the loss observed by the regret minimizer $\mathcal{R}_{\Delta}$ employing \eqref{eq:omwu}. Next, we continue with the proof of \cref{lemma:high-order-smooth}. For the convenience of the reader, the statement of the lemma is included below.

\highsmooth*

\begin{proof}[Proof of \cref{lemma:high-order-smooth}]
First of all, we know that 
\begin{equation}
    \label{eq:utility-game}
    \vec{\ell}_i^{(t)}[a_i] = \sum_{a_{i'} \in \range{n_{i'}}, i' \neq i} \Lambda(a_1, \dots, a_i, \dots, a_m) \prod_{i' \neq i} \vec{x}^{(t)}_{i'}[a_{i'}],
\end{equation}
where recall that by assumption $\Lambda(\cdot) \in [0,1]$. In particular, given that $\vec{L}^{(t)}_{i}[\edge{j}{k}] = \vec{x}^{(t)}_i[j](\vec{\ell}^{(t)}_i[k] - \vec{\ell}^{(t)}_i[j])$, we may conclude that 
\begin{equation}
    \label{eq:L-loss}
    \vec{L}_i^{(t)}[\edge{j}{k}] = \sum_{a_{i'} \in \range{n_{i'}}, i' \neq i} \mleft( \Lambda(a_1, \dots, k, \dots, a_m) - \Lambda(a_1, \dots, j, \dots, a_m) \mright) \vec{x}^{(t)}_i[j] \prod_{i' \neq i} \vec{x}^{(t)}_{i'}[a_{i'}].
\end{equation}
However, we know from the equivalence shown in \cref{theorem:equivalence} that $\vec{x}_{i'}^{(t)}[a_{i'}] = \sum_{\tree_{i'} \in \treeset_{i', a_{i'}}} \vec{X}_{i'}^{(t)}[\tree_{i'}]$. Given that the set of directed trees with different roots are disjoint, \eqref{eq:L-loss} can be expressed as
\begin{equation}
    \label{eq:LX-loss}
    \vec{L}_i^{(t)}[\edge{j}{k}] = \sum_{\tree_{i'} \in \treeset_{i'}} \Lambda'(\tree_{1}, \dots, \tree_m) \prod_{i' \in \range{m}} \vec{X}^{(t)}_{i'}[\tree_{i'}],
\end{equation}
where $\Lambda': \treeset_1 \times \dots \times \treeset_m \to [-1, 1]$. In particular, this implies that 
\begin{align}
    \left\lvert (D_h \vec{L}_i)^{(t)}[\edge{j}{k}] \right\rvert &= \left\lvert \sum_{s=0}^h \binom{h}{s} (-1)^{h - s} \vec{L}_i^{(t + s)}[\edge{j}{k}] \right\rvert \label{eq:align-binom} \\
    &= \left\lvert \sum_{\tree_{i'} \in \treeset_{i'}} \Lambda'(\tree_1, \dots, \tree_m) \sum_{s=0}^h \binom{h}{s} (-1)^{h - s} \prod_{i' \in \range{m}} \vec{X}^{(t+s)}_{i'}[\tree_{i'}] \right\rvert \label{eq:align-Lloss} \\
    &\leq \sum_{\tree_{i'} \in \treeset_{i'}} \left\lvert  \sum_{s=0}^h \binom{h}{s} (-1)^{h - s} \prod_{i' \in \range{m}} \vec{X}^{(t+s)}_{i'}[\tree_{i'}] \right\rvert \label{eq:align-triangle} \\
    &= \sum_{\tree_{i'} \in \treeset_{i'}} \left\lvert \mleft( D_h \mleft( \prod_{i' \in \range{m}} \vec{X}_{i'}[\tree_{i'}] \mright) \mright)^{(t)} \right\rvert \label{eq:bound-Li},
\end{align}
where \eqref{eq:align-binom} uses the equivalent formulation \eqref{eq:binomial} of $h$-order finite differences; \eqref{eq:align-Lloss} follows from \eqref{eq:LX-loss}; \eqref{eq:align-triangle} follows from the triangle inequality and the fact that $|\Lambda'(\cdot)| \in [0, 1]$; and the final line uses again the equivalent formulation of finite differences of \eqref{eq:binomial}, with the convention that $\prod_{i' \in \range{m}} \vec{X}_{i'}[\tree_{i'}]$ refers to the sequence $\prod_{i' \in \range{m}} \vec{X}^{(1)}_{i'}[\tree_{i'}], \dots, \prod_{i' \in \range{m}} \vec{X}^{(T)}_{i'}[\tree_{i'}]$. Next, it will be convenient to assume that $\vec{X}_i^{(0)}$ is the uniform distribution over the simplex $\Delta^{|\treeset_i|}$, as well as $\bigL_i^{(0)} = \bigL_i^{(-1)} = \vec{0}$. By construction, for every player $i \in \range{m}$ the vector $\vec{X}_i^{(t)}$ is updated using \eqref{eq:omwu} under the sequence of losses $\bigL^{(1)}, \dots, \bigL^{(T)}$, implying that for each $\tree \in \treeset_i$,
\begin{equation}
    \label{eq:expanded-OMWU}
    \vec{X}_i^{(t_0 + t + 1)}[\tree] = \frac{\exp \mleft\{ \eta \mleft( \bigL_i^{(t_0 - 1)}[\tree] - \sum_{s=0}^t \bigL_i^{(t_0 + s)}[\tree] - \bigL_i^{(t_0 + t)}[\tree] \mright) \mright\} \vec{X}_i^{(t_0)}[\tree]}{\sum_{\tree' \in \treeset_i} \exp \mleft\{\eta \mleft( \bigL_i^{(t_0 - 1)}[\tree'] - \sum_{s=0}^t \bigL_i^{(t_0 + s)}[\tree'] - \bigL_i^{(t_0 + t)}[\tree'] \mright) \mright\} \vec{X}_i^{(t_0)}[\tree']}.
\end{equation}
For notational convenience, we will let $\bar{\bigL}^{(t)}_{i, t_0} = \bigL_i^{(t_0 - 1)} - \sum_{s=0}^{t-1} \bigL_i^{(t_0 + s)} - \bigL_i^{(t_0 + t - 1)}$, for $0 \leq t_0 \leq T$ and $t \geq 0$. Moreover, for a vector $\vec{z} \in \R^{|\treeset_i|}$ we define the following function:
\begin{equation}
    \label{eq:softmax-phi}
    \phi_{t_0, \tree}(\vec{z}) = \frac{\exp\{\vec{z}[\tree]\}}{\sum_{\tree' \in \treeset_i'} \vec{X}_i^{(t_0)}[\tree'] \exp \mleft\{\vec{z}[\tree']\mright\}}.
\end{equation}
Equipped with this notation, we can equivalently write \eqref{eq:expanded-OMWU} as $\vec{X}_{i}^{(t_0 + t)}[\tree] = \vec{X}_{i}^{(t_0)}[\tree] \phi_{t_0, \tree}\mleft( \eta \bar{\bigL}^{(t)}_{i, t_0} \mright)$ for $t \geq 1$. In particular, this implies that for any $i' \in \range{m}$ and $\tree_{i'} \in \treeset_{i'}$,

\begin{equation}
    \label{eq:prod-phi}
    \prod_{i' \in \range{m}} \vec{X}_{i'}^{(t_0 + t)} [\tree_{i'}] = \prod_{i' \in \range{m}} \vec{X}_{i'}^{(t_0)}[\tree_{i'}] \phi_{t_0, \tree_{i'}} \mleft( \eta \bar{\bigL}^{(t)}_{i', t_0} \mright).
\end{equation}

Before we proceed with the analysis, let us introduce the \emph{shift operator}. Specifically, for a sequence of vectors $\vec{z} = (\vec{z}^{(1)}, \dots, \vec{z}^{(T)})$, the s-shifted sequence, denoted with $E_s \vec{z}$, is such that $(E_s \vec{z})^{(t)} = \vec{z}^{(t + s)}$, for $1 \leq t \leq T - s$. With this notation, we observe that
\begin{equation*}
    \mleft( D_1 \bar{\bigL}_{i, t_0} \mright)^{(t)} = \bigL_i^{(t_0 + t - 1)} - 2 \bigL_i^{(t_0 + t)} = \bigL_i^{(t_0 + t - 1)} - 2 (E_1 \bigL_i)^{(t_0 + t - 1)},
\end{equation*}
which in particular implies that for any $h' \geq 1$,
\begin{equation}
    \label{eq:shift-recurs}
    \mleft( D_{h'} \bar{\bigL}_{i, t_0} \mright)^{(t)} = (D_{h'-1} \bigL_i)^{(t_0 + t - 1)} - 2 (E_1 D_{h'-1} \bigL_i)^{(t_0 + t - 1)}.
\end{equation}

Next we proceed with bounding $\bar{\bigL}_{i, t_0}^{(t)}$, for any $0 \leq t_0 \leq T$ and $t \geq 0$. In particular, for a fixed $\tree \in \treeset_i$, we know that $\bigL_i^{(t)}[\tree] = \sum_{(j, k) \in E(\tree)} \vec{x}_i^{(t)}[j] (\vec{\ell}_i^{(t)}[k] - \vec{\ell}_i^{(t)}[j])$. Thus, given that $\vec{\ell}_i^{(t)}[j] \in [0, 1]$, for any $j \in \range{n_i}$ and $t \geq 0$, as follows from \eqref{eq:utility-game}, we can conclude that there exists $\tree \in \treeset_i$ such that 
\begin{equation}
    \label{eq:bigL-infty}
    \| \bigL_i^{(t)}\|_{\infty} = \left\lvert \sum_{(j, k) \in E(\tree)} \vec{x}_i^{(t)}[j](\vec{\ell}_i^{(t)}[k] - \vec{\ell}_i^{(t)}[j]) \right\rvert \leq \sum_{(j, k) \in E(\tree)} \vec{x}_i^{(t)}[j] \left\lvert \vec{\ell}_i^{(t)}[k] - \vec{\ell}_i^{(t)}[j] \right\rvert \leq 1,
\end{equation}
where we used that $\left\lvert \vec{\ell}_i^{(t)}[k] - \vec{\ell}_i^{(t)}[j] \right\rvert \leq 1$, as well as the fact that $\tree$ is a directed tree, implying that every node has at most one outgoing (directed) edge (recall \cref{definition:directed_tree}). Along with the triangle inequality, this implies that 
\begin{equation}
    \label{eq:L-bar-infty}
    \| \bar{\bigL}_{i, t_0}^{(t)} \|_{\infty} = \mleft\| \bigL_i^{(t_0 - 1)} - \sum_{s=0}^{t-1} \bigL_i^{(t_0 + s)} - \bigL_i^{(t_0 + t - 1)} \mright\|_{\infty} \leq (t+2).
\end{equation}
Before we proceed with the next claim, it will be useful to introduce the sequence $\mathfrak{L}^{(t)}$, defined as
\begin{equation*}
    \mathfrak{L}^{(t)} \coloneqq \mleft( \eta \bar{\bigL}_{1, t_0}^{(t)}, \eta \bar{\bigL}_{2, t_0}^{(t)}, \dots, \eta \bar{\bigL}_{m, t_0}^{(t)} \mright).
\end{equation*}

\begin{claim}
    \label{claim:D_h-bound}
Consider some parameter $\alpha \in (0, 1)$ such that $\mleft\| \mleft( D_{h'} \mleft( \eta \bar{\bigL}_{i, t_0} \mright) \mright)^{(t)} \mright\|_{\infty} \leq \frac{1}{B_1} \alpha^{h'} (h')^{B_0 h'}$, for all $i \in \range{m}$, $0 \leq h' \leq h$, and $t \in \range{h + 1 - h'}$, where $B_1 = 12 e^5 m$ and $B_0 \geq 3$. Then, for any $0 \leq t_0 \in T - h - 1$ and $\tree_1 \in \treeset_1, \dots, \tree_m \in \treeset_m$, it holds that
\begin{equation}
    \left\lvert \mleft( D_h \mleft( \prod_{i' \in \range{m}} \vec{X}_{i'}[\tree_{i'}] \mright) \mright)^{(t_0 + 1)} \right\rvert \leq \alpha^h h^{B_0 h + 1} \prod_{i' \in \range{m}} \vec{X}_{i'}^{(t_0)}(\tree_{i'}).
\end{equation}
\end{claim}
\begin{proof}
We will apply \cref{lemma:boundedness-chain-rule} with $n \coloneqq \sum_{i' \in \range{m}} |\treeset_{i'}|$, time horizon $h + 1$, the sequence $\mathfrak{L}^{(t)}$, and the function $\phi$ that maps the concatenation of $\vec{z}_{i'} \in \R^{|\treeset_{i'}|}$, for all $i' \in \range{m}$, to the function $\prod_{i'} \phi_{t_0, \tree_{i'}}(\vec{z}_{i'})$, where $\phi_{t_0, \tree_{i'}}$ was defined in \eqref{eq:softmax-phi}; that is, 
\begin{equation}
    \label{eq:def-phi}
    \phi_{t_0}(\vec{z}_1, \dots, \vec{z}_m) \coloneqq \prod_{i' \in \range{m}} \phi_{t_0, \tree_{i'}}(\vec{z}_{i'}).
\end{equation}
In this context, let us verify the conditions of \cref{lemma:boundedness-chain-rule}. First, \cite[Lemma B.6]{Daskalakis21:Near} implies that the function $\phi_{t_0}$ is $(1, e^3m)$-bounded (in the sense of \cref{lemma:boundedness-chain-rule}), and $B_1 = 12 e^5 m \geq 2e^2 (e^3 m)$. Moreover, by \cite[Lemma B.7]{Daskalakis21:Near} we know that each function $\phi_{t_0, \tree_{i'}}$ has radius of convergence---with respect to the origin $\vec{0}$---greater than $1$, and hence, the radius of convergence of $\phi_{t_0}$ is also greater than $1$. We also know, by assumption, that $\mleft\| \mleft( D_{h'} \mathfrak{L} \mright)^{(t)} \mright\|_{\infty} \leq \frac{1}{B_1} \alpha^{h'} (h')^{B_0 h'}$, for all $0 \leq h' \leq h$ and $t \in \range{h + 1 - h'}$. As a result, \cref{lemma:boundedness-chain-rule} implies that 
\begin{equation}
    \label{eq:applylemma}
    \left\lvert\mleft(  D_h \mleft( \phi \circ \mathfrak{L} \mright) \mright)^{(t)} \right\rvert \leq \frac{12 e^5 m}{B_1} \alpha^h h^{B_0 h + 1}.
\end{equation}
Finally, we have that 
\begin{align}
    \frac{1}{\prod_{i' \in \range{m}} \vec{X}^{(t_0)}_{i'}[\tree_{i'}] } \left\lvert \mleft( D_h \mleft( \prod_{i' \in \range{m}} \vec{X}_{i'}[\tree_{i'}] \mright) \mright)^{(t_0 + 1)} \right\rvert &= \left\lvert \mleft( D_h \mleft( \prod_{i' \in \range{m}} \mleft( \phi_{t_0, \tree_{i'}} \circ \mleft( \eta \bar{\bigL}_{i', t_0} \mright) \mright) \mright) \mright)^{(1)} \right\rvert \label{eq:align-def-phi} \\
    &= \left\lvert \mleft( D_h \mleft(\phi_{t_0} \circ \mleft( \eta \bar{\bigL}_{1, t_0}, \dots, \eta \bar{\bigL}_{m, t_0} \mright) \mright) \mright)^{(1)} \right\rvert \label{eq:bound_lemma} \\
    &= \left\lvert D_h \mleft( \phi_{t_0} \circ \mathfrak{L} \mright)^{(1)} \right\rvert \notag \\
    & \leq \frac{12 e^5 m}{B_1} \alpha^h h^{B_0 h + 1} = \alpha^h h^{B_0 h + 1} \label{eq:end_bound-frac},
\end{align}
where \eqref{eq:align-def-phi} follows from \cref{eq:prod-phi}; \eqref{eq:bound_lemma} simply uses the definition of $\phi_{t_0}$ in \eqref{eq:def-phi}; and \eqref{eq:end_bound-frac} follows from \eqref{eq:applylemma} (which in turn is a consequence of \cref{lemma:boundedness-chain-rule}), as well as the fact that $B_1 = 12 e^5 m$. This completes the proof.
\end{proof}

\begin{claim}
    \label{claim:inductive_step}
Fix some $1 \leq h \leq H$, a parameter $\alpha \in (0, 1)$, and assume that the learning rate $\eta$ is such that $\eta \leq \min \mleft\{ \frac{\alpha}{36 e^5 m}, \frac{1}{12e^5 (H+3) m} \mright\}$. Moreover, assume that for all $0 \leq h' < h$, $t \leq T - h'$, and $i \in \range{m}$, it holds that $\| (D_{h'} \bigL_i)^{(t)}\|_{\infty} \leq \alpha^{h'} (h' + 1)^{B_0(h' + 1)}$. Then, for all $t \in \range{T - h}$ and $i \in \range{m}$,
\begin{equation}
    \mleft\| \mleft( D_h  \bigL_{i} \mright)^{(t)} \mright\|_{\infty} \leq \alpha^{h} h^{B_0 h + 1}.
\end{equation}
\end{claim}

\begin{proof}
Let us set $B_1 \coloneqq 12 e^5 m$, so that $\eta \leq \min \mleft\{ \frac{\alpha}{3 B_1}, \frac{1}{B_1(H+3)} \mright\}$. We know from \eqref{eq:L-bar-infty} that for $t + 2 \leq h + 3$ it follows that
\begin{equation}
    \label{eq:base}
    \mleft\| D_0 \mleft( \eta \bar{\bigL}_{i, t_0} \mright)^{(t)}\mright\|_{\infty} = \eta \mleft\| \bar{\bigL}_{i, t_0}^{(t)} \mright\| \leq \eta (t+2) \leq \frac{1}{B_1},
\end{equation}
where we used the fact that $\eta \leq 1/(B_1(H+3))$. Next, for $1 \leq h' \leq h$,
\begin{align}
    \mleft\| \mleft( D_{h'} \mleft( \eta \bar{\bigL}_{i, t_0} \mright) \mright)^{(t)} \mright\|_{\infty} &\leq \eta \mleft\| \mleft( D_{h' - 1} \bigL_i \mright)^{(t_0 + t - 1)} \mright\|_{\infty} + 2 \eta \mleft\| \mleft( D_{h' - 1} \bigL_i \mright)^{(t_0 + t)} \mright\|_{\infty} \label{eq:align-trian-shift} \\
    &\leq 3 \eta \alpha^{h' - 1} (h')^{B_0 h'} \label{eq:inductive-hypothesis} \\
    &\leq \frac{1}{B_1} \alpha^{h'} (h')^{B_0 h'}, \label{eq:align-small-eta}
\end{align}
where \eqref{eq:align-trian-shift} follows from \eqref{eq:shift-recurs} and the triangle inequality; \eqref{eq:inductive-hypothesis} is a consequence of the assumption in the claim; and \eqref{eq:align-small-eta} uses the fact that $\eta \leq \frac{\alpha}{3 B_1}$. Next, given that $\bigL^{(t)}_i[\tree] = \sum_{(j, k) \in E(\tree)} \vec{x}_i^{(t)}[j] (\vec{\ell}_i^{(t)}[k] - \vec{\ell}_i^{(t)}[j])$, where recall the expression of $\vec{\ell}_i^{(t)}$ in \eqref{eq:utility-game}, we can infer that 
\begin{align}
    \bigL_i^{(t)} &= \sum_{a_{i'} \in \range{n_{i'}}} \hat{\Lambda}(a_1, \dots, a_m) \prod_{i' \in \range{m}} \vec{x}_{i'}[a_{i'}] \notag \\
    &= \sum_{\tree_{i'} \in \treeset_{i'}} \Tilde{\Lambda}(\tree_1, \dots, \tree_m) \prod_{i' \in \range{m}} \vec{X}_{i'}[\tree_{i'}], \label{eq:utility-bigL}
\end{align}
where $\hat{\Lambda}$ and $\Tilde{\Lambda}$ are functions such that $\hat{\Lambda}(\cdot) \in [-1, 1]$ and $\Tilde{\Lambda}(\cdot) \in [-1, 1]$. More precisely, \eqref{eq:utility-bigL} is obtained from \cref{theorem:equivalence}. As a result, similarly to \eqref{eq:bound-Li} we may conclude that 
\begin{equation}
    \left\lvert (D_h \bigL_i)^{(t)}[\tree] \right\rvert \leq  \sum_{\tree_{i'} \in \treeset_{i'}} \left\lvert \mleft( D_h \mleft( \prod_{i' \in \range{m}} \vec{X}_{i'}[\tree_{i'}] \mright) \mright)^{(t)} \right\rvert.
\end{equation}

The next step is to invoke \cref{claim:D_h-bound} in order to bound the induced term. Specifically, by \eqref{eq:base} and \eqref{eq:align-small-eta} we see that its conditions are met, from which we can conclude that for $t \in \range{T - h}$
\begin{align}
    \mleft\| \mleft( D_h \bigL_i \mright)^{(t)} \mright\|_{\infty} &\leq \sum_{\tree_{i'} \in \treeset_{i'}} \left\lvert \mleft( D_h \mleft( \prod_{i' \in \range{m}} \vec{X}_{i'}[\tree_{i'}] \mright) \mright)^{(t)} \right\rvert \label{eq:util-bigL} \\
    &\leq \alpha^h h^{B_0 h + 1} \sum_{\tree_{i'} \in \treeset_{i'}} \prod_{i' \in \range{m}} \vec{X}^{(t-1)}_{i'}[\tree_{i'}] \label{align-cor-lemma} \\
    &= \alpha^h h^{B_0 h + 1}, \label{eq:align-final-norm}
\end{align}
where \eqref{eq:util-bigL} follows from \eqref{eq:utility-bigL}; \eqref{align-cor-lemma} is an immediate application of \cref{claim:D_h-bound}; and \eqref{eq:align-final-norm} follows from the fact that each $\vec{X}_{i'}$ is a probability distribution over the space of directed trees $\treeset_{i'}$, for all $i' \in \range{m}$, and as such the induced product distribution normalizes to $1$.
\end{proof}
Finally, it follows from \eqref{eq:bigL-infty} that $\mleft\| \mleft( D_0 \bigL_i \mright)^{(t)} \mright\|_{\infty} \leq 1 = \alpha^0 1^{B_0 1}$, for all $i \in \range{m}$. Thus, we can inductively invoke \cref{claim:inductive_step} for $B_0 = 3$ to infer that for all $0 \leq h \leq H$, $i \in \range{m}$, and $t \in \range{T - h}$ that $\mleft\| \mleft( D_h \bigL_i \mright)^{(t)} \mright\|_{\infty} \leq \alpha^h h^{3h + 1}$, as long as $\eta \leq \frac{\alpha}{36 e^5 m}$ and $\alpha \leq 1/(H + 3)$. Finally, following the argument given in the proof of \cref{claim:inductive_step}, we obtain that for any $0 \leq h' \leq h$,
\begin{equation*}
    \mleft\| \mleft( D_{h'} \mleft( \eta \bar{\bigL}_{i, t_0} \mright) \mright)^{(t)} \mright\|_{\infty} \leq \frac{1}{B_1} \alpha^{h'} (h')^{3 h'}.
\end{equation*}
Thus, we can invoke \cref{claim:D_h-bound} to conclude that for any $0 \leq t_0 \leq T - h - 1$ and $\tree_{1} \in \treeset_{1}, \dots, \tree_m \in \treeset_{m}$,
\begin{equation*}
    \left\lvert \mleft( D_h \mleft( \prod_{i' \in \range{m}} \vec{X}_{i'}[\tree_{i'}] \mright) \mright)^{(t_0 + 1)} \right\rvert \leq \alpha^h h^{3 h + 1} \prod_{i' \in \range{m}} \vec{X}_{i'}^{(t_0)}(\tree_{i'}).
\end{equation*}
As a result, plugging-in this bound into \eqref{eq:bound-Li} we obtain that for $t \in \range{T - h}$,
\begin{align*}
    \mleft\| \mleft( D_h \vec{L}_i \mright)^{(t)} \mright\|_{\infty} &\leq \sum_{\tree_{i'} \in \treeset_{i'}} \left\lvert \mleft( D_h \mleft( \prod_{i' \in \range{m}} \vec{X}_{i'}[\tree_{i'}] \mright) \mright)^{(t)} \right\rvert \\
    &\leq \alpha^h h^{3 h + 1} \sum_{\tree_{i'} \in \treeset_{i'}} \prod_{i' \in \range{m}} \vec{X}^{(t-1)}_{i'}[\tree_{i'}], \\
    &= \alpha^h h^{3 h + 1},
\end{align*}
completing the proof.
\end{proof}

The final technical ingredient is the following lemma, which can be shown analogously to \cite[Lemma 4.2]{Daskalakis21:Near} in our setting.

\begin{restatable}{lemma}{backind}
    \label{lemma:bound-variance}
    There are universal constants $C, C' \geq 1$ so that for a time horizon $T \geq 4$ and $H \coloneqq \lceil \log T \rceil$, if all players employ $\slomwu$ with learning rate $\eta$ such that $1/T \leq \eta \leq \frac{1}{C m H^4}$, then, 
    \begin{equation}
        \label{eq:check}
        \sum_{t=1}^T \var_{\vec{p}_i^{(t)}} \mleft( \vec{L}_i^{(t)} - \vec{L}_i^{(t-1)} \mright) \leq \frac{1}{2} \var_{\vec{p}_i^{(t)}} \mleft( \vec{L}_i^{(t-1)} \mright) + C' H^5,
    \end{equation}
    for any $i \in \range{m}$.
\end{restatable}
%
%\backind*
%
\begin{proof}

Recall the aforementioned tool due to \cite{Daskalakis21:Near}.

\lgbld*

We hope to apply Lemma \ref{lem:general-bound-l-dl} with $n=n_i$, $\vec{P}\^t = \vec{p}_i\^t$ and $\vec{Z}\^t = \vec{L}_i\^t$, as well as $\zeta = 7\eta$.  To do so, we must verify that the preconditions of the lemma hold.  Set $C_1 = 8256$ (note that $C_1$ is the constant appearing in item \ref{it:gen-close}).  Our assumption that $\eta \leq \frac{1}{C \cdot mH^4}$ implies that, as long as the constant $C$ satisfies $C \geq 4^4 \cdot 7 \cdot C_1 = 14794752$, 
\begin{equation}
  \label{eq:sl-eta-reqs}
\eta \leq \min \left\{ \frac{\alpha^4}{7C_1}, \frac{\alpha_0}{36e^5 m}\right\}.
  \end{equation}

We verify precondition \ref{it:gen-down} from \cref{lemma:high-order-smooth}, which implies that for all $i \in \range{m}$, $0 \leq h \leq H$, and $1 \leq t \leq T - h$, it holds that 
\begin{equation*}
    %\label{eq:high-order-smooth}
    \mleft\| \mleft( D_h \vec{L}_i \mright)^{(t)} \mright\|_{\infty} \leq h \mleft(\alpha_0 h^3 \mright)^{h} \leq H \mleft( \alpha_0 H^3 \mright)^h.
\end{equation*}

To verify precondition \ref{it:gen-close}, we first confirm that our selection of $\zeta = 7\eta$ places it in the desired interval $[1/(2T), \alpha^4/C_1]$ as $\eta \leq \frac{\alpha^4}{7C_1}$.  We use the following claim.

\begin{claim}[Multiplicative Stability of (O)MWU]
    \label{claim:mult-stabl}
    For every player $i \in \range{m}$ the sequence $\vec{p}_i^{(1)}, \dots, \vec{p}_i^{(T)}$ is $7\eta$-consecutively close.
\end{claim}
\begin{proof}
We know from the update rule of \eqref{eq:omwu} that 
\begin{equation*}
    \vec{p}_i^{(t+1)}[\edge{j}{k}] = \frac{\exp\{- \eta (2 \vec{L}_i^{(t)}[\edge{j}{k}] - \vec{L}_i^{(t-1)}[\edge{j}{k}] )\}}{\sum_{a \neq b} \exp\{ - \eta (2 \vec{L}_i^{(t)}[\edge{a}{b}] - \vec{L}_i^{(t-1)}[\edge{a}{b}])\} \vec{p}_i^{(t)}[\edge{a}{b}]} \vec{p}_i^{(t)}[\edge{j}{k}].
\end{equation*}
Thus, since $\| \vec{L}_i^{(t)} \|_{\infty} \leq 1$, it follows that 
\begin{align*}
    \vec{p}_i^{(t+1)}[\edge{j}{k}] \leq \frac{\exp\{3 \eta\}}{\sum_{a \neq b} \exp\{ - 3\eta  \} \vec{p}_i^{(t)}[\edge{a}{b}]} \vec{p}_i^{(t)}[\edge{j}{k}] &= \exp \{ 6 \eta \} \vec{p}_i^{(t)}[\edge{j}{k}],
\end{align*}
and similar reasoning yields that $\vec{p}_i^{(t+1)}[\edge{j}{k}] \geq \exp \{ - 6 \eta \} \vec{p}_i^{(t)}[j \to k]$. Finally, the claim follows given that $\exp \{ 6 \eta\} \leq 1 + 7 \eta$ for the $\eta$ considered in the lemma.
\end{proof}

Therefore, Lemma \ref{lem:general-bound-l-dl} applies and we have
\begin{align*}
    \sum_{t=1}^T  \Var{\vec{p}_i\^t}{\vec{L}_i\^t- \vec{L}_i\^{t-1}} &\leq 2\alpha \sum_{t=1}^T \Var{\vec{p}_i\^t}{\vec{L}_i\^{t-1}} + 165120(1+7\eta) H^5+2\\
    &\leq \frac{1}{2} \cdot \sum_{t=1}^T \Var{\vec{p}_i\^t}{\vec{L}_i\^{t-1}} + C' H^5
\end{align*}
for $C' = 2 + 165120(1+7/8256) = 165262$, as desired.  This completes the proof of the lemma.
\end{proof}

Finally, we are ready to establish our main theorem, namely \Cref{theorem:main summary}. Before with proceed with the proof, let us first restate the theorem for the convenience of the reader.

\intmain*

\begin{proof}%[Proof of \cref{theorem:main}]
First of all, we know from \cref{theorem:convergence_correlated} that it suffices to show that the internal regret of every player $i \in \range{m}$ is bounded by $O(m \log n_i \log^4 T)$. In particular, by \cref{proposition:internal-external} this reduces to showing that the external regret of the regret minimizer $\mathcal{R}_{\Delta}$ employed in \cref{algo:irm} operating over the pairs of actions is bounded by $O(m \log n_i \log^4 T)$. But the latter can be easily verified by \cref{lemma:RVU++} and \cref{lemma:bound-variance} (as in the proof of \cite[Theorem 3.1]{Daskalakis21:Near}).
\end{proof}

\section{Analysis of the Blum-Mansour Algorithm}
\label{appendix:proofs-BM}
\subsection{Near-Optimal Bounds for Swap Regret}

%This work hinges on the recent findings of \citep{Daskalakis21:Near} demonstrating that all players running the online learning algorithm \Opthedge to make decisions in a repeated general-sum game will experience ${\rm poly}(\log T)$ external regret over $T$ time steps.  We combine this with the classic result of \citet{Blum07:From} establishing an algorithm with minimal swap-regret using an external-regret-minimizing algorithm as a black-box.

Let us first recall our main theorem regarding algorithm $\bmomwu$.

\thmbm*
\begin{proof}
    We start with the main lemma of \citep{Daskalakis21:Near}, bounding the external regret of \Opthedge in even the adversarial setting
    
    \begin{lemma}
      \label{lem:reg-variance}
      There is a constant $C > 0$ so that the following holds. 
    Suppose we follow the \Opthedge updates with step size $0 < \eta < 1/C$, for an arbitrary sequence of losses $\vectorLL\^1, \ldots, \vectorLL\^T \in [0,1]^{n}$. Then for any vector $\vectorecks^\st \in \Delta^{n}$, it holds that
    \begin{align}
      \label{eq:adv-var-bound-2}
          \sum_{t=1}^T \lng \vectorecks\^t - \vectorecks^\st, \vectorLL\^t \rng \leq \frac{\log n}{\eta} + \sum_{t=1}^T \left(\frac{\eta}{2} + C \eta^2\right)  \Var{\vectorecks\^t}{\vectorLL\^t - \vectorLL\^{t-1}} - \sum_{t=1}^T  \frac{(1-C\eta)\eta}{2} \cdot \Var{\vectorecks\^{t}}{\vectorLL\^{t-1}}.
    \end{align}
    \end{lemma}
    
    We can use this bound to demonstrate that the Blum-Mansour (\BM) algorithm achieves good swap-regret \citep{Blum07:From}.  Recall that the \BM algorithm maintains $n$ copies of a no-regret algorithm $\mathcal{R}_{\Delta,1}, \ldots, \mathcal{R}_{\Delta,n}$, (we choose \Opthedge), where each algorithm $\mathcal{R}_{\Delta,g}$ produce iterates $\matrixKYU\^t[g,\rowdot] = (\matrixKYU\^t[g,1],\ldots,\matrixKYU\^t[g,n]) \in \Delta^n$ at time step $t$. The \BM algorithm then chooses $\vectorecks\^t \in \Delta^n$ satisfying $\vectorecks\^t = \vectorecks\^t \matrixKYU\^t$ where $\matrixKYU\^t$ is a matrix with $g^\text{th}$ row $\matrixKYU\^t[g,\rowdot]$.  Upon receiving a loss vector $\vectorLL\^t$, the \BM algorithm experiences loss of $\lng \vectorecks\^t, \vectorLL\^t \rng$ and distributes the loss vector $\vectorecks\^t[g] \cdot \vectorLL\^t$ to $\mathcal{R}_{\Delta,g}$ for each $g \in \range{n}$, which use them to update and produces iterates $\matrixKYU\^{t+1}[g,\rowdot]$. By Lemma \ref{lem:reg-variance}, each algorithm $\mathcal{R}_{\Delta,g}$ experiences regret of
    \begin{align}
      & \sum_{t= 1}^T \lng \matrixKYU\^t[g,\rowdot] - \vec{q}^\st, \vectorecks\^t[g] \cdot \vectorLL\^t \rng \nonumber\\
      \leq & \frac{\log n}{\eta} + \sum_{t=1}^T \left( \frac{\eta}{2} + C \eta^2 \right) \Var{\matrixKYU\^t[g,\rowdot]}{\vectorecks\^t[g]\vectorLL\^t - \vectorecks\^{t-1}[g] \vectorLL\^{t-1}} - \sum_{t=1}^T \frac{(1-C\eta)\eta}{2} \cdot \Var{\matrixKYU\^t[g,\rowdot]}{\vectorecks\^{t-1}[g] \vectorLL\^{t-1}}\nonumber
    \end{align}
    Thus the overall regret of \BM is bounded as follows: for any swap function $\phi : \range{n} \ra \range{n}$, we have
    \begin{align}
      & \sum_{t=1}^T \lng \vectorecks\^t, \vectorLL\^t \rng \nonumber\\
      =& \sum_{t=1}^T \lng \vectorecks\^t \matrixKYU\^t, \vectorLL\^t \rng \nonumber\\
      =& \sum_{t=1}^T \sum_{g=1}^n \lng \matrixKYU\^t[g,\rowdot], \vectorecks\^t[g] \vectorLL\^t \rng \nonumber\\
      \leq & \sum_{g=1}^n  \sum_{t=1}^T \left( \vectorecks\^t[g] \cdot \vectorLL\^t[\phi( g )] +  \frac{\log n}{\eta}  + \left( \frac{\eta}{2} + C \eta^2\right)  \Var{\matrixKYU\^t[g,\rowdot]}{\vectorecks\^t[g] \cdot \vectorLL\^t - \vectorecks\^{t-1}[g]\cdot  \vectorLL\^{t-1}}\right. \nonumber\\
      &\qquad \qquad \left. -   \frac{(1-C\eta)\eta}{2} \cdot \Var{\matrixKYU\^t[g,\rowdot]}{\vectorecks\^{t-1}[g] \cdot \vectorLL\^{t-1}}\right),
    \end{align}
    which implies that
    \begin{align}
      &  \sum_{t=1}^T \lng \vectorecks\^t, \vectorLL\^t \rng  - \sum_{t=1}^T \sum_{g=1}^n  \vectorecks\^t[g] \cdot \vectorLL\^t[\phi( g )] \nonumber\\
      \leq & \frac{n \log n}{\eta} + \sum_{t=1}^T \sum_{g=1}^n  \left( \frac{\eta}{2} + C \eta^2\right)  \Var{\matrixKYU\^t[g,\rowdot]}{\vectorecks\^t[g] \cdot \vectorLL\^t - \vectorecks\^{t-1}[g]\cdot  \vectorLL\^{t-1}} \nonumber\\
      & \qquad \quad - \sum_{t=1}^T \sum_{g=1}^n  \frac{(1-C\eta)\eta}{2} \cdot \Var{\matrixKYU\^t[g,\rowdot]}{\vectorecks\^{t-1}[g] \cdot \vectorLL\^{t-1}}\nonumber.
    \end{align}
    
    Thus, we want to show the following lemma
    
    \begin{lemma}
      \label{lem:bound-l-dl}
      % There is a sufficiently large constant $C > 1$ so that the following holds.
      Suppose all players play according to \BM-\Opthedge with step size $\eta$ satifying $1/T \leq \eta \leq \frac{1}{C \cdot mn^3\log^4(T)}$ for a sufficiently large constant $C$. Then for any player $i \in \range{m}$ and any $g \in \range{n}$, the overall losses for player $i$: $\vectorLL_i\^1, \ldots, \vectorLL_i\^T \in \BR^{n}$, the weight player $i$ places on strategy $g$: $\vectorecks\^1[g], \ldots, \vectorecks\^T[g] \in \BR$, and the strategy vectors output by player $i$'s $g^{\text{th}}$ instance of \Opthedge $\mathcal{R}_{\Delta,i,g}$: $\matrixKYU_i[g,\rowdot]\^1, \ldots, \matrixKYU_i[g,\rowdot]\^T \in \BR^{n}$ satisfy
    \begin{align}
      \label{eq:var-l-dl-main}
      &\sum_{t=1}^T  \Var{\matrixKYU\^t[g,\rowdot]}{\vectorecks\^t[g] \cdot \vectorLL\^t - \vectorecks\^{t-1}[g]\cdot  \vectorLL\^{t-1}}\\
      &\leq \frac12 \cdot \sum_{t=1}^T \Var{\matrixKYU\^t[g,\rowdot]}{\vectorecks\^{t-1}[g] \cdot \vectorLL\^{t-1}} + O\paren{\log^5 (T)}
    \end{align}
    \end{lemma}
    Consequently, combining these two results gives, for $\eta \in \ps{\frac1T, \frac{1}{C \cdot mn^3\log^4(T)}}$,
    \begin{align}
      &  \swapreg_i^T \nonumber\\
      \leq & \frac{n \log n}{\eta} + \sum_{t=1}^T \sum_{g=1}^n  \left( \frac{\eta}{2} + C \eta^2\right)  \Var{\matrixKYU\^t[g,\rowdot]}{\vectorecks\^t[g] \cdot \vectorLL\^t - \vectorecks\^{t-1}[g]\cdot  \vectorLL\^{t-1}} \nonumber\\
      & \qquad \quad - \sum_{t=1}^T \sum_{g=1}^n  \frac{(1-C\eta)\eta}{2} \cdot \Var{\matrixKYU\^t[g,\rowdot]}{\vectorecks\^{t-1}[g] \cdot \vectorLL\^{t-1}}\nonumber\\
      \leq & \frac{n \log n}{\eta}+ \sum_{g=1}^n\paren{\frac{2\eta}{3} \sum_{t=1}^T   \Var{\matrixKYU\^t[g,\rowdot]}{\vectorecks\^t[g] \cdot \vectorLL\^t - \vectorecks\^{t-1}[g]\cdot  \vectorLL\^{t-1}} -\frac{\eta}{3} \sum_{t=1}^T \Var{\matrixKYU\^t[g,\rowdot]}{\vectorecks\^{t-1}[g] \cdot \vectorLL\^{t-1}}}\nonumber\\
      \leq &\frac{n \log n}{\eta}+\frac{2n\eta}{3} O\paren{ \log^5 (T)}.
    \end{align}
    Hence, by setting $\eta = \frac{1}{C \cdot mn^3\log^4(T)}$ we recover the bound stated in \cref{thm:bound-l-dl}.
\end{proof}
The main technical tool for achieving this result is Lemma \ref{lem:dh-bound}, establishing high order smoothness under the fixed point iterates of \BM-\Opthedge.

\subsection{Tools for \BM-\Opthedge Swap Regret}

Recall from earlier our definition of rooted trees and the Markov Chain Tree Theorem.  We re-iterate those results here under a slightly different formulation, catering to the proof techniques for this section.\\

\begin{definition}
Suppose $Z$ is an $n \times n$ matrix and $G = (V,E)$ with $V = \range{n}$ is the weighted directed graph associated with $Z$. We say a subgraph $T$ of $G$ is a directed tree rooted at $i \in \range{n}$ if (1) $T$ does not contain any cycles and (2) Node $i$ has no outgoing edges, while every other node $j \in \range{n}$ has exactly one outgoing edge. For each node $i \in \range{n}$, we write $\MT_i$ to denote the set of all directed trees rooted at node $i$.
\end{definition}

We let $\BR^{n \times n}$ (respectively, $\BC^{n \times n}$) denote the space of real-valued (respectively, complex-valued) $n \times n$ matrices. For each $j \in \range{n}$, we introduce the following functions $\Phi_{\MCT, j} : \BC^{n \times n} \ra \BC$:
\begin{align}
\Phi_{\MCT,j} (Z) :=  \frac{\sum_{T \in \MT_j} \exp \left( \sum_{(a,b) \in T} Z_{ab} \right)}{\sum_{j =1}^n \sum_{T \in \MT_j} \exp \left( \sum_{(a,b) \in T} Z_{ab} \right) }\nonumber.
\end{align}
Note that for $Z \in \BR^{n \times n}$, we have $\Phi_{\MCT,j}(Z) \in \BR$. 

\begin{theorem}[Markov chain tree theorem]
  \label{thm:mct}
Let $\matrixKYU \in \BR^{n \times n}$ be a Markov chain so that $\matrixKYU_{ij} > 0$ for all $i,j \in \range{n}$. Then the stationary distribution of $Q$ is given by the vector $(\Phi_{\MCT,1}(\ln Q), \ldots, \Phi_{\MCT, 1}(\ln Q))$, where $\ln Q$ denotes the matrix whose $(i,j)$ entry is $\ln \matrixKYU_{ij}$. 
\end{theorem}

\begin{lemma}
  \label{lem:mct-taylor}
  Fix any $Z\^0 \in \BR^{n \times n}$. Consider any function of the form $\phi(Z) := \frac{\Phi_{\MCT,j}(Z\^0 +  Z)}{\Phi_{\MCT,j}(Z\^0)}$. % , where $\beta \in (0, 1/(\pi n))$.
  Then the sum of the Taylor series coefficients of order $k \geq 0$ of $\phi$ at $\bbzero$ is bounded above by $30 \cdot (2n^3)^k$, and has radius of convergence greater than $1/n$. % \noah{todo prove the ROC bound}
\end{lemma}
\begin{proof}
  Fix any multi-index $\gamma \in \BZ_{\geq 0}^{n \times n}$; we will bound $\frac{d^\gamma}{dZ^\gamma} \phi(\bbzero)$. Fix any $Z \in \BC^{n \times n}$ so that $\| Z \|_\infty \leq \pi/(3n)$. Set $\zeta = Z\^0 +  Z$. Note that, for any $T \in \MT_j$, 
  \begin{align}
    \left| \exp \left( \sum_{(a,b) \in T} \zeta_{ab} \right) \right| = &  \left| \exp \left( \sum_{(a,b) \in T} Z\^0_{ab} + \sum_{(a,b) \in T} (\zeta_{ab} - Z\^0_{ab}) \right) \right| \nonumber\\
    \leq & \exp(\pi/3) \cdot \exp \left( \sum_{(a,b) \in T} Z\^0_{ab} \right)\nonumber,
  \end{align}
  where the inequality uses the fact that $\left| \sum_{(a,b) \in T} (\zeta_{ab} - Z\^0_{ab}) \right| \leq \pi/3$ and that $Z\^0$ is real-valued.

  For any $a \in \BR$ and $\zeta \in \BC$ with $|\zeta| \leq \pi/3$, we have
  \begin{align}
\Re(\exp(a + \zeta)) = \exp(a) \cdot \Re(\zeta) \geq \exp(a) \cdot \cos(\pi/3) \cdot \exp(-\pi/3) > \exp(a)/10.\nonumber
  \end{align}
  Thus, for any $j' \in \range{n}$ and $T \in \MT_{j'}$, it holds that
  \begin{align}
    \Re \left( \exp\left(\sum_{(a,b) \in T} \zeta_{ab} \right) \right) = & \Re\left(  \exp \left( \sum_{(a,b) \in T} Z\^0_{ab} + \sum_{(a,b) \in T} (\zeta_{ab} - Z\^0_{ab}) \right) \right)\nonumber\\
    \geq & \frac{1}{10} \cdot \exp \left( \sum_{(a,b) \in T} Z\^0_{ab} \right)\nonumber.
  \end{align}
  Thus, since $Z\^0$ is real-valued,
  \begin{align}
    |\phi(Z)| =& \frac{|\Phi_{\MCT,j}(\zeta)|}{\Phi_{\MCT,j}(Z\^0)} \nonumber\\
    \leq &  \frac{\sum_{j'=1}^n \sum_{T \in \MT_{j'}} \exp \left( \sum_{(a,b) \in T} Z\^0_{a,b}\right)}{\sum_{T \in \MT_j} \exp \left( \sum_{(a,b) \in T} Z\^0_{a,b}\right)} \cdot \frac{ \sum_{T \in \MT_j} \left| \exp \left( \sum_{(a,b) \in T} Z_{ab} \right) \right|}{
           \sum_{j' = 1}^n \sum_{T \in \MT_{j'}} \Re\left(  \exp \left( \sum_{(a,b) \in T} Z_{ab} \right) \right) }\nonumber\\
    \leq &  \frac{\sum_{j'=1}^n \sum_{T \in \MT_{j'}} \exp \left( \sum_{(a,b) \in T} Z\^0_{a,b}\right)}{\sum_{T \in \MT_j} \exp \left( \sum_{(a,b) \in T} Z\^0_{a,b}\right)} \cdot  \frac{\pi/3 \cdot \sum_{T \in \MT_j} \exp \left( \sum_{(a,b) \in T} Z\^0_{a,b} \right)}{1/10 \cdot \sum_{j'=1}^n \sum_{T \in \MT_{j'}} \exp \left( \sum_{(a,b) \in T} Z\^0_{a,b} \right)}\nonumber\\
    = & \exp(\pi/3) \cdot 10 < 30 \nonumber.
  \end{align}
  By the multivariate version of Cauchy's integral formula,
  \begin{align}
    \left| \frac{d^\gamma}{dZ^\gamma} \phi(\bbzero) \right| =& \left| \frac{\gamma!}{(2\pi i)^{n^2}} \int_{|Z_{11}| = \pi/(3n)} \cdots \int_{|Z_{nn}| = \pi/(3n)} \frac{\phi(Z)}{\prod_{j_1, j_2 \in \range{n}} (Z_{j_1j_2})^{\gamma_{j_1j_2} + 1}} d\zeta_{11} \cdots d\zeta_{nn}\right| \nonumber\\
    \leq & 30 \cdot \frac{\gamma!}{(2\pi)^{n^2}} \cdot \frac{1}{(\pi/(3n))^{n^2+|\gamma|}} \cdot \left( 2\pi \cdot \pi/(3n)\right)^{n^2} = 30 \cdot \gamma! \cdot (3n/\pi)^{|\gamma|}\label{eq:agamma-bound}.
  \end{align}
  For any integer $k \geq 0$, the number of tuples $\gamma \in \BZ_{\geq 0}^{n^2}$ with $|\gamma| = k$ is ${k + n^2 - 1 \choose n^2 - 1} \leq (2n^2)^k$. Thus, if we write $\phi(Z) = \sum_{\gamma \in \BZ_{\geq 0}^{n^2}} a_\gamma \cdot z^\gamma$, it follows that $a_\gamma = \frac{1}{\gamma!} \cdot \frac{d^\gamma}{dZ^\gamma} \phi(\bbzero)$, and so for each $k \geq 0$, $\sum_{\gamma : |\gamma| = k} |a_\gamma| \leq 30 \cdot (2n^2)^k \cdot (3n/\pi)^k \leq 30 \cdot (2n^3)^k$.

  To see the lower bound on the radius of convergence, note that \Cref{eq:agamma-bound} gives that for each $\gamma \in \BZ_{\geq 0}^{n^2}$, letting $k := |\gamma|$, we have $|a_\gamma|^{1/k} \leq 30^{1/k} \cdot (3n/\pi)$, which tends to $3n/\pi < n$ as $k \ra \infty$. Thus, by the multivariate version of the Cauchy-Hadamard theorem, the radius of convergence of the power series of $\phi$ at $\bbzero$ is greater than $1/n$.
\end{proof}

In the statement of Lemma \ref{lem:fds-stationary} below, the quantity $0^0$ is interpreted as 1 (in particular, $(h')^{B_0 h'} = 1$ for $h' = 0$).

\begin{lemma}
  \label{lem:fds-stationary}
  Fix any $B_1 \geq 4e^2 n^3,\ B_0 \geq 3$.
  Consider a sequence $\matrixKYU\^0, \ldots, \matrixKYU\^h \in \BR^{n \times n}$ of ergodic Markov chains, so that $\| \log \matrixKYU\^0 - \log \matrixKYU\^t \|_\infty \leq \frac{1}{B_1}$ for $0 \leq t \leq h$. Suppose that for some $\alpha \in (0,1)$, for each $1 \leq h' \leq h$ and $0 \leq t \leq h-h'$, it holds that $\| \fds{h'}{(\ln \mat{Q})}{t} \|_\infty \leq \frac{1}{B_1} \cdot \alpha^{h'} \cdot (h')^{B_0 h'}$. Then, if $p\^0, \ldots, p\^h \in \Delta^n$ denotes the sequence of stationary distributions for $\matrixKYU\^0, \ldots, \matrixKYU\^h$, it holds that, for any $j \in \range{n}$,
  \begin{align}
    % \left\| \fds{h}{p}{t} \right\|_\infty
    \frac{\left| \fds{h}{p[j]}{t} \right| }{p\^0[j]}
    \leq & \frac{720 n^3 e^2}{B_1} \cdot \alpha^h \cdot h^{B_0 h + 1}\nonumber.
  \end{align}
\end{lemma}
\begin{proof}
  By the Markov chain tree theorem we have that, for each $j \in \range{n}$, $p\^t[j] = \Phi_{\MCT,j}(\ln \matrixKYU\^t)$. We now apply Lemma \ref{lemma:boundedness-chain-rule} with $T = h+1$, $\mat{z}\^t = \ln \matrixKYU\^{t-1} - \ln \matrixKYU\^0$ for $1 \leq t \leq h+1$, $R_1 = 1,\ R_2 = 2n^3$, and the values of $B_0, B_1$ given in the hypothesis of this lemma (Lemma \ref{lem:fds-stationary}). Moreover, we set $\phi(Z) = \frac{\Phi_{\MCT,j}((\ln \matrixKYU\^0) + Z)}{\Phi_{\MCT,j}(\ln \matrixKYU\^0)} = \frac{\Phi_{\MCT,j}((\ln \matrixKYU\^0) + Z)}{p\^0[j]}$. Lemma \ref{lem:mct-taylor} gives that the function $\phi$ is $(30, 2n^3)$-bounded, and has radius of convergence greater than $1/n$ at the point $Z = \bbzero$. Then the hypotheses of Lemma \ref{lem:fds-stationary} imply those of Lemma \ref{lemma:boundedness-chain-rule}, and Lemma \ref{lemma:boundedness-chain-rule} gives that
  \begin{align}
\frac{\left| \fds{h}{p[j]}{t} \right|}{p\^0[j]} = \frac{\left| \fds{h}{(\phi \circ (\ln \matrixKYU - \ln \matrixKYU\^0))}{0} \right|}{p\^0[j]} \leq \frac{720 n^3 e^2}{B_1} \cdot \alpha^h \cdot h^{B_0 h + 1}\nonumber.
  \end{align}
  (Here we have used that $\phi(\ln \matrixKYU\^t - \ln \matrixKYU\^0) = \frac{\Phi_{\MCT,j}(\ln \matrixKYU\^t)}{p\^0[j]} = p\^t[j] / p\^0[j]$.)
\end{proof}
%\noah{Minor modifications should give the same bound on $\vectorLL_1$ norm without an extra $n$ factor.}

\begin{lemma}
  \label{lem:softmax-lip-bound}
  For $n \in \BN$, let $\xi_1, \ldots, \xi_n \geq 0$ so that $\xi_1 + \cdots + \xi_n = 1$. For any $j \in \range{n}$, the function
  \begin{align}
\phi_j((z_1, \ldots, z_n)) = \frac{\exp(z_j)}{\sum_{k=1}^n \xi_k \cdot \exp(z_k)}\nonumber
  \end{align}
  satisfies, for any $z \in \BR^n$ with $\| z \|_\infty \leq 1/4$,
  \begin{align}
    | \log \phi_j(z)| \leq \| z \|_\infty \leq 6 \| z \|_\infty\nonumber.
    \end{align}
  \end{lemma}
  \begin{proof}
    For $0 \leq x \leq 1$, we have $1 + x \leq \exp(x) \leq 1 + 2x$. Then, for $\| z \|_\infty \leq 1/2$,
    \begin{align}
\phi_j(z) \leq \frac{1 + 2z_j}{\sum_{k=1}^n \xi_k \cdot (1 + z_k)} \leq \frac{1 + 2 \| z \|_\infty}{1 - \| z \|_\infty} \leq (1 + 2 \| z \|_\infty)^2 \nonumber
    \end{align}
    and
    \begin{align}
\phi_j(z) \geq \frac{1 + z_j}{\sum_{k=1}^n \xi_k \cdot (1 + 2z_k)} \geq \frac{1 - \| z \|_\infty}{1 + 2 \| z \|_\infty} \geq (1 - \| z \|_\infty)(1 - 2 \| z \|_\infty)\nonumber.
    \end{align}
    Thus, for $\| z \|_\infty \leq 1/4$,
    \begin{align}
-6 \| z \|_\infty \leq \log(1 - \| z \|_\infty) + \log(1 - 2 \| z \|_\infty) \leq \log \phi_j(z) \leq &  2 \log(1 + 2 \| z \|_\infty) \leq 4 \| z \|_\infty\nonumber.
    \end{align}
  \end{proof}

\begin{lemma}\label{lem:log-softmax-ak-bound}
% For $n \in \BN$, let $\phi : \BR^n \ra \BR$ be any softmax-type function, i.e., for some $\xi_1, \ldots, \xi_n \geq 0$ so that $\xi_1 + \cdots + \xi_n = 1$, and some $j \in \range{n}$,
For $n \in \BN$, let $\xi_1, \ldots, \xi_n \geq 0$ such that $\xi_1 + \cdots + \xi_n = 1$.  For each $j\in \range{n}$, define $\phi_j : \BR^n \ra \BR$ to be the function
\begin{align}
\phi_j((z_1, \ldots, z_n)) = \frac{\xi_j \exp(z_j)}{\sum_{k =1}^n \xi_k \cdot \exp(z_k)}\nonumber
\end{align}
and let $P_{\log \phi_j}(z) = \sum_{\gamma \in \BZ_{\geq 0}^n} a_{j,\gamma} \cdot z^\gamma$ denote the Taylor series of $\log \phi_j$. Then for any $j \in \range{n}$ and any integer $k \geq 1$, 
\begin{align}
    \sum_{\gamma \in \BZ_{\geq 0}^n : \ |\gamma| = k} \bigcard{a_{j,\gamma}} \leq e^{k}/k\nonumber.
\end{align}
\end{lemma}

\begin{proof}
We have that
\begin{equation*}
    \frac{\partial \log \phi_j}{\partial z_{t}} = \frac{1}{\phi_j}\cdot\frac{\partial \phi_j}{\partial z_{t}} = \begin{cases}
        -\phi_t &\text{ if $t \ne j$}\\
        1-\phi_j &\text{ if $t = j$}
    \end{cases}
\end{equation*}
and so,
\begin{align}
    \sum_{\gamma \in \BZ_{\geq 0}^n : \ |\gamma| = k} \bigcard{a_{j,\gamma}}  &=\frac{1}{k!}\sum_{t \in \range{n}^k} \bigcard{\frac{\partial^k \log \phi_j(0)}{\partial z_{t_1} \partial z_{t_2}\cdots \partial z_{t_k}}} \nonumber\\
    &=\frac{1}{k!}\sum_{t_1 \in \range{n}}\sum_{t_{-1} \in \range{n}^{k-1}} \bigcard{\frac{\partial^{k-1} }{\partial z_{t_2} \cdots \partial z_{t_k}}\frac{\partial \log \phi_j(0)}{\partial z_{t_1}}}\nonumber\\
    &=\frac{1}{k!}\sum_{t_1 \in \range{n}}\sum_{t_{-1} \in \range{n}^{k-1}} \bigcard{\frac{\partial^{k-1} }{\partial z_{t_2} \cdots \partial z_{t_k}}\paren{\phi_{t_1} - \mathbbm{1}[t_1 = j]}(0)}
\end{align}

For $k=1$,
\begin{align*}
    \frac{1}{k!}\sum_{t_1 \in \range{n}} \bigcard{\phi_{t_1}(0) - \mathbbm{1}[t_1 = j]} \leq 1+\sum_{t_1 \in \range{n}} \xi_{t_1} = 2
\end{align*}

For $k \geq 2$, the constant $\mathbbm{1}[t_1 = j]$ will be eliminated by the derivative, and we will have

\begin{align}
    \label{eq:apply-lem}
    \frac{1}{k!}\sum_{t_1 \in \range{n}}\sum_{t_{-1} \in \range{n}^{k-1}} \bigcard{\frac{\partial^{k-1} \phi_{t_1}(0) }{\partial z_{t_2} \cdots \partial z_{t_k}}} &\leq \frac{1}{k!}\sum_{t_1 \in \range{n}} \paren{(k-1)! \cdot \xi_{t_1}e^k}\\
    &=\frac{1}{k!}\cdot (k-1)! \cdot e^k \cdot \sum_{t_1 \in \range{n}} \xi_{t_1} = e^k/k \nonumber
\end{align}
where \eqref{eq:apply-lem} comes from the following lemma due to \citep{Daskalakis21:Near}.

\begin{lemma}[\citep{Daskalakis21:Near}]
% For $n \in \BN$, let $\phi : \BR^n \ra \BR$ be any softmax-type function, i.e., for some $\xi_1, \ldots, \xi_n \geq 0$ so that $\xi_1 + \cdots + \xi_n = 1$, and some $j \in \range{n}$,
For $n \in \BN$, let $\xi_1, \ldots, \xi_n \geq 0$ such that $\xi_1 + \cdots + \xi_n = 1$.  For each $j\in \range{n}$, define $\phi_j : \BR^n \ra \BR$ to be the function
\begin{align}
\phi_j(z_1, \ldots, z_n) = \frac{\xi_j \exp(z_j)}{\sum_{k =1}^n \xi_k \cdot \exp(z_k)}\nonumber
\end{align}
and let $P_{\phi_j}(z) = \sum_{\gamma \in \BZ_{\geq 0}^n} a_{j,\gamma} \cdot z^\gamma$ denote the Taylor series of $\phi_j$. Then for any $j \in \range{n}$ and any integer $k \geq 1$, 
\begin{align}
    \sum_{\gamma \in \BZ_{\geq 0}^n : \ |\gamma| = k} \bigcard{a_{j,\gamma}} \leq \xi_j e^{k+1}\nonumber.
\end{align}
\end{lemma}

\end{proof}

\subsection{Proof of \BM High Order Smoothness}
\label{sec:bm-dh-bound-proof}
\textbf{Lemma \ref{lem:dh-bound}} [Detailed]
  \textit{Fix a parameter $\alpha \in \left(0, \frac{1}{H+3} \right)$. 
  If all players follow BM-\Opthedge updates with step size $\eta \leq \frac{1}{311040 e^7H^2n_i^3}$ %\leq \frac{\alpha}{36 e^5 m}$
  , then for any player $i \in \range{m}$, integer $h$ satisfying $0 \leq h \leq H$, time step $t \in [T-h]$, and $g \in [n_i]$, it holds that}
\begin{align}
\| \fds{h}{(\vectorecks_i[g] \cdot \vectorLL_i)}{t} \|_\infty \leq \alpha^h \cdot h^{3h+1}\nonumber.
\end{align}

% \iffalse
%   the following are true:
%   \begin{enumerate}
%   \item If $\eta \leq \frac{1}{C H^2 \cdot e^{Cm}}$, then letting $\alpha = \eta \cdot CH^2 e^{Cm} \in (0,1)$, it holds that for $0 \leq h \leq H$ and $t \in [T-h]$,
%     \begin{align}
% \| \fds{h}{\vectorLL_i}{t} \|_\infty \leq \alpha^h \cdot h^{3h+1} \cdot e^{Cm} \nonumber.
%     \end{align}
%   \item If $\eta \leq \frac{1}{C^2 H^2}$, then letting $\alpha = \eta \cdot C^2 H^2$, it holds that for $0 \leq h \leq H$ and $t \in [T-h]$,
%     \begin{align}
%       \| \fds{h}{\vectorLL_i}{t} \|_\infty \leq \alpha^h \cdot h^{3h \log m + 1} \cdot Cm^7.\nonumber
%     \end{align}
%   \end{enumerate}
%   \fi
%   \begin{equation}
% \| \fds{h}{\vectorLL_i}{t} \|_\infty \leq O( n_{\max}\cdot \eta) ^h \cdot h^{O(h)}.\label{eq:fdslt-ub}
% \end{equation}
% \noah{set step size $\eta$ somewhere, define $B_0,\alpha$ earlier...}

\begin{proof}

We prove this inductively, showing that, for all $i\in \range{m},0\leq h \leq H,t \in [T-h],g\in [n_i]$,

\begin{align*}
    \left\| \fds{h}{\vectorLL_i}{t} \right\|_\infty \leq\ & \alpha^h \cdot h^{3h+1}\\
    \left\| \fds{h}{\vectorecks_i[g]\cdot \vectorLL_i}{t} \right\|_\infty \leq\ & \alpha^h \cdot h^{3h+1}\\
    \left\| \fds{h}{\vectorecks_i}{t} \right\|_\infty \leq\ & \alpha^h \cdot h^{3h+1}.
\end{align*}

The base case of $h=0$ is evident from the fact that losses and strategy probabilities are in $[0,1]$, and therefore $\| \vectorLL_i\^{t} \|_\infty, \| \vectorecks_i\^{t} \|_\infty \leq 1$.  So, proving the following inductive statement is sufficient to prove the lemma.

%Since $\alpha^h \cdot h^{B_0 h + 1} \leq \alpha^h \cdot (h+1)^{B_0 (h + 1)}$, the conclusion of the inductive step for $h'$ will be sufficient for the inductive hypothesis at step $h'+1$.  Additionally,
%\noah{todo move this somewhere better} 

%We will use \eqref{eq:q-form} to prove the following claim. %  by induction. \noah{statement of this claim has to change}

\begin{claim}
  \label{clm:upwards-main}
  Fix some $1 \leq h \leq H$. Suppose that for some $B_0 \geq 3$ and all integers $h'$ satisfying $1 \leq h' < h$, all $i \in \range{m}$, all $t \leq T-h'$, and all $g \in [n_i]$, it holds that
  \begin{align}
  \label{eq:inductive-loss-step}
  \| \fds{h'}{\vectorLL_i}{t} \|_\infty \leq\ &  \alpha^{h'} \cdot (h')^{B_0h'+1}\\
  \label{eq:inductive-sl-step}
    \left\| \fds{h'}{\vectorecks_i[g]\cdot \vectorLL_i}{t} \right\|_\infty \leq\ & \alpha^{h'} \cdot (h')^{B_0h'+1}\\
    \label{eq:inductive-strat-step}
    \| \fds{h'}{\vectorecks_i}{t} \|_\infty \leq\ & \alpha^{h'} \cdot (h')^{B_0 h'+1 }.
    \end{align}
  Suppose that the step size $\eta$ satisfies $\eta \leq \frac{1}{311040 e^7H^2n_i^3}$. Then for all $i \in \range{m}$ and $1 \leq t \leq T-h$,
  \begin{align}
    \label{eq:inductive-loss-formal}
    \left\| \fds{h}{\vectorLL_i}{t} \right\|_\infty \leq\ & \alpha^h \cdot h^{B_0h+1}\\
    \label{eq:inductive-sl-formal}
    \left\| \fds{h}{\vectorecks_i[g]\cdot \vectorLL_i}{t} \right\|_\infty \leq\ & \alpha^h \cdot h^{B_0h+1}\\
    \label{eq:inductive-strat-formal}
    \left\| \fds{h}{\vectorecks_i}{t} \right\|_\infty \leq\ & \alpha^h \cdot h^{B_0h+1}.
  \end{align}
%   \noah{something else involving $x$?}
\end{claim}

% \iffalse
% \begin{claim}
%   Assume for all $i \in \range{m}$ and all $t \leq T-h'$, it holds that $\| \vectorLL_i\^{t} \|_\infty, \| x_i\^{t} \|_\infty \leq 1$ and for all $1 \leq h' < h$,
% $$\| \fds{h'}{\vectorLL_i}{t} \|_\infty \leq \alpha^{h'} \cdot (h')^{B_0 h'+1 }$$
% and
% $$\| \fds{h'}{x_i}{t} \|_\infty \leq \alpha^{h'} \cdot (h')^{B_0 h'+1 }$$
% Suppose that the step size $\eta$ satisfies $\eta \leq $\maxfish{add this}.  % \footnote{Here we use the convention that $0^0 = 1$; in particular, $\| \vectorLL_i\^t \|_\infty \leq 1$ for all $i,t$.}
%   Then for all $i \in \range{m}$ and $1 \leq t \leq T-h$, 
%   \begin{equation}
%     \label{eq:inductive-loss-formal-old}
% \left\| \fds{h}{\vectorLL_i}{t}  \right\|_\infty \leq  \alpha^h \cdot h^{B_0 h + 1}
% \end{equation}
% and
% \begin{equation}
%     \label{eq:inductive-strat-formal-old}
% \left\| \fds{h}{x_i}{t}  \right\|_\infty \leq  \alpha^h \cdot h^{B_0 h + 1}.
% \end{equation}
% \end{claim}
% \fi

\begin{proof}[Proof of Claim \ref{clm:upwards-main}]

First, notice that for any agent $i \in \range{m}$, any \Opthedge instance $\mathcal{R}_{\Delta,i,g}$ of agent $i$ with $g \in [n_i]$, any $t_0 \in \{0, 1, \ldots, T \}$, and any $t \geq 0$, by the definition (\ref{eq:omwu}) of the \Opthedge updates, it holds that, for each $j \in [n_i]$, 
 \begin{align*}
&\matrixKYU_i\^{t_0+t+1}[g,j]\\
&= \frac{\matrixKYU_i\^{t_0}[g,j] \cdot \exp\left( \eta \cdot \left( \vectorecks_i\^{t_0-1}[g] \vectorLL_i\^{t_0-1}[j] - \sum_{s=0}^t \vectorecks_i\^{t_0+s}[g]\vectorLL_i\^{t_0+s}[j] - \vectorecks_i\^{t_0+t}[g]\vectorLL_i\^{t_0+t}[j]\right) \right)}{\sum_{k=1}^{n_i} \matrixKYU_i\^{t_0}[g,k] \cdot \exp\left( \eta \cdot \left( \vectorecks_i\^{t_0-1}[g] \vectorLL_i\^{t_0-1}(k) - \sum_{s=0}^t \vectorecks_i\^{t_0+s}[g] \vectorLL_i\^{t_0+s}(k) - \vectorecks_i\^{t_0+t}[g] \vectorLL_i\^{t_0+t}(k)\right) \right)}
\end{align*}
where $\matrixKYU_i^t[g,j]$ denotes the weight placed on action $j$ by algorithm $\mathcal{R}_{\Delta,i,g}$ at time $t$.  We can define edge values $\vectorLL_i\^0, \vectorLL_i\^{-1}, \matrixKYU_i\^0$ to ensure that the above equation holds even for $t_0 \in \{0,1\}$.
Now, for any $g,j \in [n_i]$, any integer $t_0$ satisfying $0 \leq t_0 \leq T$; and any integer $t \geq 0$, let us define
% \noah{todo deal with $t_0 = 1$ (need an $\vectorLL_i\^0$...)}
$$
\matrixLL_{i,t_0}\^{t}[g,j] := \vectorecks_i\^{t_0-1}[g] \vectorLL_i\^{t_0-1}[j] - \sum_{s=0}^{t-1} \vectorecks_i\^{t_0+s}[g]\vectorLL_i\^{t_0+s}[j] - \vectorecks_i\^{t_0+t-1}[g]\vectorLL_i\^{t_0+t-1}[j]
$$
Also, for a vector $z = (z[1], \ldots, z[n_i]) \in \BR^{n_i}$ and indices $g,j \in [n_i]$, define
\begin{equation}
  \label{eq:def-phij}
\phi_{t_0,g,j}(z) := \frac{ \exp\left(  z[j]\right)}{\sum_{k=1}^{n_i} \matrixKYU_i\^{t_0}[g,k] \cdot \exp\left( z(k)\right)},
\end{equation}
so that 
\begin{equation}\label{eq:q-form}
\matrixKYU_i\^{t_0+t}[g,j] = \matrixKYU_i\^{t_0}[g,j]\cdot \phi_{t_0,g,j}(\eta \cdot \matrixLL_{i,t_0}\^{t}[g,\rowdot])
\end{equation}
for $t \geq 1$, where $\matrixLL_{i,t_0}\^{t}[g,\rowdot]$ denotes the vector $(\matrixLL_{i,t_0}\^{t}[g,1],\ldots, \matrixLL_{i,t_0}\^t[g,n_i])$.\\

Next, note that, for all $g,j \in [n_i]$,
$$\fds{1}{\matrixLL_{i,t_0}[g,j]}{t} = \vectorecks_i\^{t_0+t-1}[g]\vectorLL_i\^{t_0+t-1}[j] - 2\vectorecks_i\^{t_0+t}[g]\vectorLL_i\^{t_0+t}[j]$$
and so, for any $1 \leq h' \leq h$,
\begin{align}
\left|\fds{h'}{\paren{\eta \cdot \matrixLL_{i,t_0}[g,j]}}{t}\right| &= \left|\fds{h'-1}{\paren{\eta \cdot \vectorecks_i[g]\cdot \vectorLL_i[j]}}{t_0+t-1} - 2\fds{h'-1}{\paren{\eta \cdot \vectorecks_i[g]\cdot \vectorLL_i[j]}}{t_0+t}\right|\nonumber \\
&\leq 3\eta \cdot \alpha^{h'-1}\cdot(h'-1)^{B_0h'} \label{eq:BM-bar-bound}
\end{align}
where \eqref{eq:BM-bar-bound} follows from our inductive hypothesis \eqref{eq:inductive-sl-step}.  % Similarly, 
% $$\fds{h'}{\paren{\eta \cdot \matrixLL_{i,t_0}[g,j]}}{t} \leq 3\eta \cdot \frac{12e^2}{B_1} \cdot \alpha^{h'-1}\cdot(h'-1)^{B_0h'}$$
% for all $h' \leq h$, as our argument is inductive on $h$.

Additionally, since $\bigcard{\matrixLL_{i,t_0}[g,j]} \leq t+2$ for all $t_0, i, g, j$, we have $\bigcard{\eta \cdot \matrixLL_{i,t_0}[g,j]} \leq \eta \cdot (t+2) \leq 1$ for $\eta \leq \frac{1}{t+2}$. By Lemma \ref{lem:log-softmax-ak-bound}, the function $z \mapsto \log \phi_{t_0,g,j}(z)$ is $(1,e)$-bounded, and so for each $1 \leq h' \leq h$ we may apply Lemma \ref{lemma:boundedness-chain-rule} with $h = h'$, $\vec{z}\^t = \eta \cdot \matrixLL_{i,t_0}\^t[g,\rowdot]$ and $B_1 = \frac{1}{\eta} \cdot \min \left\{ \frac{\alpha}{3}, \frac{1}{H+3} \right\} = \frac{\alpha}{3\eta}$.  %\noah{make sure last equality holds}. 
% Thus, $\eta \cdot \matrixLL_{i,t_0}[g,j]$ satisfies the preconditions of
Thus, %by Lemma \ref{lemma:boundedness-chain-rule}
we can conclude, for all $1 \leq h' \leq h$, $t \leq h+1$, and $g \in [n_i]$,
\begin{equation}\label{eq:log-phi-bound}
\fds{h'}{\log\paren{\phi_{t_0,g,j}(\eta \cdot \matrixLL_{i,t_0}\^{t}[g,\rowdot])}}{t} \leq 12e^3 \cdot 3\eta \cdot \alpha^{h'} \cdot (h')^{B_0h'+1} \leq 36 e^3 H \eta \cdot \alpha^{h'} \cdot (h')^{B_0h'}.
\end{equation}
% by Lemma \ref{lem:log-softmax-ak-bound}.  % Therefore, $\frac1h \log\paren{\phi_{t_0,g,j}(\eta \cdot \matrixLL_{i,t_0}\^{t}[g,\rowdot])}$ satisfies the preconditions of Lemma \ref{lemma:boundedness-chain-rule}.
% and Lemma \ref{lem:expand-pow-seq}, we have \maxfish{we can't directly apply lemma 4.5 because the precondition is slightly different. that is, we have $D_h$ bounded by $ \alpha^{h'} \cdot (h'+1)^{B_0 (h'+1) }$ not $ \alpha^{h'} \cdot (h')^{B_0 (h') }$.  If there is a way to abridge this by writing 4.5 more generally, I'm down.}
%\begin{align*}
%\fd{h-1}{\paren{x_i[g]\cdot \vectorLL_i[j]}}&= \sum_{\pi: \range{h-1} \to [2]} \paren{\shf{t'_{\pi,1}}{\fd{h'_{\pi,1}}{x_i[g]}}}\paren{\shf{t'_{\pi,2}}{\fd{h'_{\pi,2}}{\vectorLL_i[j]}}}\\
%&\leq \sum_{h'=0}^{h-1} {h-1 \choose h'} \| \fd{h'}{x_i[g]} \|_\infty \| \fd{h-1-h'}{\vectorLL_i[j]} \|_\infty\\
%&\leq \sum_{h'=0}^{h-1} {h-1 \choose h'} \alpha^{h-1} (h'+1)^{B_0 (h'+1)} (h-h')^{B_0 (h-h')}\\
%&\leq 2\alpha^{h-1} \sum_{h'=0}^{\lceil (h-1)/2 \rceil} \frac{(h-1)^{h-1}}{h'^{h'}\paren{h-h'-1}^{h-h'-1}} (h'+1)^{B_0 (h'+1)} (h-h')^{B_0 (h-h')}\\
%&\leq 2\alpha^{h-1} \sum_{h'=0}^{\lceil (h-1)/2 \rceil} \frac{h^{h}}{(h'+1)^{h'+1}\paren{h-h'}^{h-h'}}(h'+1)^{B_0 (h'+1)} (h-h')^{B_0 (h-h')}\\
%\end{align*}
%$\fds{h'-1}{\vectorLL_i[j]}{t},\fds{h'-1}{x_i[g]}{t} \leq \alpha^{h'-1} \cdot (h')^{B_0 (h') }$ for all $g,j \in [n_i]$.  Thus, from Lemma \ref{lem:expand-pow-seq}, 

%\frac{(h-1)^{h-1}}{h'^{h'}\paren{h-h'-1}^{h-h'-1}} (

Taking the logarithm of both sides of \eqref{eq:q-form}, we have
\begin{equation}\label{eq:log-q-form}
\log\paren{\matrixKYU_i\^{t_0+t}[g,j]} =  \log{\matrixKYU_i\^{t_0}[g,j]}+ \log\paren{\phi_{t_0,g,j}(\eta \cdot \matrixLL_{i,t_0}\^{t}[g,\rowdot])}.
\end{equation}
Let $\matrixKYU_i\^{t} \in \BR^{n_i \times n_i}$ be the matrix with entries $\matrixKYU_i\^{t}[g,j]$ and $P_{i,t_0}\^{t}$ be the matrix with entries $\phi_{t_0,g,j}\paren{\eta \cdot \matrixLL_{i,t_0}\^{t}[g,\rowdot]}$, for $g,j \in [n_i]$.  Then, for all $g \in [n_i]$ and all $t_0 \in [T-h]$, $t \in \range{h}$, we have
\begin{align}
\vectorecks_i\^{t_0+t}[g] &= \Phi_{\MCT, g}\paren{\log \matrixKYU_i\^{t_0+t}}\nonumber\\
% &= \Phi_{\MCT, g}\paren{\frac1h \log \matrixKYU_i\^{t_0+t}} \noah{not\ true?}\\
&= \Phi_{\MCT, g}\paren{ \log{\matrixKYU_i\^{t_0}}+ \log\paren{P_{i,t_0}\^{t}}}.\nonumber
\end{align}
Next note that, by Lemma \ref{lem:softmax-lip-bound} and $t \leq H+1$,
\begin{align}
  \left\| \log \matrixKYU_i\^{t_0+t} - \log \matrixKYU_i\^{t_0} \right\|_\infty \leq & \max_{g,j \in [n_i]} \left| \log \phi_{t_0,g,j}(\eta \cdot \matrixLL_{i,t_0}\^t[g,\rowdot]) \right| \nonumber\\
  \leq & 6 \eta \cdot (t+2) \leq 36 e^3 H \eta \label{eq:verify-dist-precondition}.
\end{align}
We now apply Lemma \ref{lem:fds-stationary} with $\matrixKYU\^t = P_{i,t_0}\^t$, for $0 \leq t \leq h$, and $1/B_1 = 36 e^3 H \eta$. The preconditions of the lemma hold from \eqref{eq:log-phi-bound}. Then Lemma \ref{lem:fds-stationary} gives that for all $g \in [n_i]$,
\begin{align}
\left| \fds{h}{\vectorecks_i[g]}{t_0+1} \right| \leq & |\vectorecks_i\^{t_0}[g]| \cdot 720 n_i^3 e^2 \cdot 36 e^3 H \eta \cdot \alpha^{h'} \cdot (h)^{B_0 h+1}\nonumber.
\end{align}
% Since $P_{i,t_0}\^{t}$ satisfies the preconditions of Lemma \ref{lem:fds-stationary}, setting $Z = \frac1h \log{\matrixKYU_i\^{t_0}}$, we have
% $$\left\| \fds{h}{x_i}{t_0}  \right\|_\infty \leq  \alpha^h \cdot h^{B_0 h + 1}$$
% as desired.
%\noah{notation: above in terms of MCT, Lemma \ref{lem:fds-stationary} in terms of stat distribution}
%\maxfish{The proof that $\vectorLL$ has bounded $D_h$ proceeds identically to how it does in the coarse correlated Lemma 4.4}
verifying the first of three desired inductive conclusions \eqref{eq:inductive-strat-formal} as long as $\eta \leq 1/(25920 e^5 H n_i^3)$.

Next, we have that for each agent $i \in \range{m}$, each $t \in [T]$, and each $a_i \in [n_i]$, $\vectorLL_i\^t(a_i) = \ExX{a_{i'} \sim \vectorecks_{i'}\^t :\ i' \neq i}{\ML_i(a_1, \ldots, a_{m})}$. Thus, for $1 \leq t \leq T$, 
  \begin{align*}
    \left|\fds{h}{\vectorLL_i[a_i]}{t}\right| =& \left|\sum_{s=0}^h {h \choose s} (-1)^{h-s} \vectorLL_i\^{t+s}[a_i] \right| \\
    =& \left|\sum_{a_{i'} \in [n_{i'}],\ \forall i' \neq i} \ML_i(a_1, \ldots, a_m) \sum_{s=0}^h {h \choose s} (-1)^{h-s}  \cdot \prod_{i' \neq i} \vectorecks_{i'}\^{t+s}[a_{i'}]\right| \nonumber\\
    \leq  & \sum_{a_{i'} \in [n_{i'}], \ \forall i' \neq i} \left| \sum_{s=0}^h {h \choose s} (-1)^{h-s} \cdot \prod_{i' \neq i} \vectorecks_{i'}\^{t+s}[a_{i'}]  \right| \nonumber\\
    =& \sum_{a_{i'} \in [n_{i'}], \ \forall i' \neq i} \left| \fds{h}{\left( \prod_{i' \neq i} \vectorecks_{i'}[a_{i'}] \right)}{t} \right|,\\
%     =& \left\| \fds{h}{\left(\bigotimes_{j \neq i} x_j\right)}{t} \right\|_1,
    \leq& 25920 e^5 H n_i^3 \eta \cdot \alpha^{h} \cdot (h)^{B_0 h+1}\sum_{a_{i'} \in [n_{i'}], \ \forall i' \neq i} \paren{\prod_{i' \ne i} \vectorecks_i^{(t)}[a_i']}\\
    =& 25920 e^5 H n_i^3 \eta \cdot \alpha^{h} \cdot (h)^{B_0 h+1} \prod_{i' \ne i} \paren{\sum_{a_{i'} \in [n_{i'}]}\vectorecks_i^{(t)}[a_i']}\\
    =& 25920 e^5 H n_i^3 \eta \cdot \alpha^{h} \cdot (h)^{B_0 h+1}
  \end{align*}
verifying the second of three desired inductive conclusions \eqref{eq:inductive-loss-formal} as long as $\eta \leq 1/(25920 e^5 H n_i^3)$.\\

Lastly, for $\eta \leq \frac{1}{12e^2H} \cdot \frac{1}{25920 e^5 H n_i^3} = \frac{1}{311040 e^7H^2n_i^3}$, we have now verified the inductive hypotheses
\begin{align*}
    \| \fds{h'}{\vectorLL_i}{t} \|_\infty &\leq 25920 e^5 H n_i^3 \eta \cdot h' \alpha^{h'} \cdot (h')^{B_0h'} \leq \frac{1}{2e^2}\alpha^{h'} \cdot (h')^{B_0h'}\\
    \| \fds{h'}{\vectorecks_i}{t} \|_\infty &\leq 25920 e^5 H n_i^3 \eta \cdot h' \alpha^{h'} \cdot (h')^{B_0 h' } \leq \frac{1}{2e^2}\alpha^{h'} \cdot (h')^{B_0h'}
\end{align*}
for all $h'$ up to and including $h$.  Thus, we can apply Lemma \ref{lemma:boundedness-chain-rule} with $n = 2$, $\phi(a,b) = ab$ (which is $(1,1)$-bounded), $B_1 = 12e^2$, and the sequence $\vec{z}\^t = (\vectorLL_i\^t[j], \vectorecks_i\^t[g])$. Therefore, for the product sequence $\vectorecks_i[g] \cdot \vectorLL_i[j]$, we have, for all $t \in [T-h+1]$
\begin{align}
\bigcard{\fds{h}{\vectorecks_i[g]\cdot \vectorLL_i[j]}{t}} \leq \alpha^{h}\cdot(h)^{B_0\cdot h+1}
\end{align}
verifying the final inductive conclusion \eqref{eq:inductive-sl-formal}, as desired.

\end{proof}
\end{proof}

\subsection{Completing the proof of Theorem \ref{thm:bound-l-dl}}
\label{sec:polylog-main-completing-proof}
Using the lemma developed in the previous sections we now can complete the proof of Lemma \ref{lem:bound-l-dl}, the last step towards the proof of \Cref{thm:bound-l-dl}.  The lemma is restated formally below.\\

\noindent \textbf{Lemma \ref{lem:bound-l-dl}} [Detailed]
  \textit{Suppose a time horizon $T \geq 4$ is given, we set $H := \lceil \log T \rceil$, and all players play according to \BM-\Opthedge with step size $\eta$ satisfying $1/T \leq \eta \leq \frac{1}{C \cdot mH^4n_i^3}$ with $C = 311040 e^7$. Then for any $i \in \range{m}$, $g \in \range{n_i}$, the losses $\vectorLL_i\^1, \ldots, \vectorLL_i\^T \in [0,1]^{n}$ for player $i$ satisfy:}
\begin{align*}
  &\sum_{t=1}^T  \Var{\matrixKYU_i\^t[g,\rowdot]}{\vectorecks_i\^t[g] \cdot \vectorLL_i\^t - \vectorecks_i\^{t-1}[g]\cdot  \vectorLL_i\^{t-1}}\\
  &\leq \frac12 \cdot \sum_{t=1}^T \Var{\matrixKYU_i\^t[g,\rowdot]}{\vectorecks_i\^{t-1}[g] \cdot \vectorLL_i\^{t-1}} + O\paren{\log^5 (T)}
\end{align*}

\begin{proof}
We make use of the following Lemma courtesy of \citep{Daskalakis21:Near}.

\lgbld*

We hope to apply Lemma \ref{lem:general-bound-l-dl} with $n=n_i$, $\vec{P}\^t = \matrixKYU_i\^t[g,\rowdot]$ and $\vec{Z}\^t = \vectorecks_i\^t[g] \cdot \vectorLL_i\^t$, as well as $\zeta = 7\eta$.  To do so, we must verify that the preconditions of the lemma hold.  Set $C_1 = 8256$ (note that $C_1$ is the constant appearing in item \ref{it:gen-close}).  Our assumption that $\eta \leq \frac{1}{C \cdot mH^4}$ implies that, as long as the constant $C$ satisfies $C \geq 4^4 \cdot 7 \cdot C_1 = 14794752$, 
\begin{equation}
  \label{eq:eta-reqs}
\eta \leq \min \left\{ \frac{\alpha^4}{7C_1}, \frac{\alpha_0}{36e^5 m}\right\}.
  \end{equation}

To verify precondition \ref{it:gen-down}, we apply Lemma \ref{lem:dh-bound} with the parameter $\alpha$ in the lemma set to $\alpha_0$: a valid selection as $\alpha_0 < \frac{1}{H+3}$.  We conclude that, for each $i \in \range{m}$, $g \in \range{n_i}$, $0 \leq h \leq H$ and $1 \leq t \leq T-h$, it holds that $\left\| \fds{h}{\vectorecks_i[g]\cdot \vectorLL_i}{t} \right\|_\infty \leq H\cdot  \left(\alpha_0 H^{3}\right)^h$ since $\eta \leq \frac{\alpha_0}{36e^5m}$ as required by the lemma.  To verify precondition \ref{it:gen-close}, we first confirm that our selection of $\zeta = 7\eta$ places it in the desired interval $[1/(2T), \alpha^4/C_1]$ as $\eta \leq \frac{\alpha^4}{7C_1}$.  By the definition of the \Opthedge updates, for all $i \in [m]$ and $1 \leq t \leq T$, we have $\max \left\{\left\| \frac{\matrixKYU_i\^t[g,\rowdot]}{\matrixKYU_i\^{t+1}[g,\rowdot]}\right\|_\infty, \left\| \frac{\matrixKYU_i\^{t+1}[g,\rowdot]}{\matrixKYU_i\^t[g,\rowdot]}\right\|_\infty\right\} \leq \exp(6\eta)$.  Thus, the sequence $\matrixKYU_i\^1[g,\rowdot], \ldots, \matrixKYU_i\^T[g,\rowdot]$ is $(7\eta)$-consecutively close (since $\exp(6\eta) \leq 1+7\eta$ for $\eta$ satisfying (\ref{eq:eta-reqs})).  Therefore, Lemma \ref{lem:general-bound-l-dl} applies and we have
\begin{align*}
    &\sum_{t=1}^T  \Var{\matrixKYU_i\^t[g,\rowdot]}{\vectorecks_i\^t[g] \cdot \vectorLL_i\^t - \vectorecks_i\^{t-1}[g]\cdot  \vectorLL_i\^{t-1}}\\
    &\leq 2\alpha \sum_{t=1}^T \Var{\matrixKYU_i\^t[g,\rowdot]}{\vectorecks_i\^{t-1}[g] \cdot \vectorLL_i\^{t-1}} + 165120(1+7\eta) H^5+2\\
    &\leq \frac{1}{2} \cdot \sum_{t=1}^T \Var{\matrixKYU_i\^t[g,\rowdot]}{\vectorecks_i\^{t-1}[g] \cdot \vectorLL_i\^{t-1}} + C' H^5
\end{align*}
for $C' = 2 + 165120(1+7/8256) = 165262$, as desired.
\end{proof}

\end{document}